\documentclass{article}
\pdfoutput=1
\usepackage{iclr2027_conference,times}

\usepackage{amsmath,amsfonts,bm}

\def\eqref#1{equation~\ref{#1}}

\def\1{\bm{1}}

\DeclareMathAlphabet{\mathsfit}{\encodingdefault}{\sfdefault}{m}{sl}
\SetMathAlphabet{\mathsfit}{bold}{\encodingdefault}{\sfdefault}{bx}{n}

\usepackage{amsmath,amssymb,amsthm,booktabs,graphicx,microtype,multirow,xcolor,hyperref,url}
\usepackage{pifont}
\usepackage{colortbl}
\definecolor{oursbg}{RGB}{210,232,255}
\definecolor{blockbg}{HTML}{F2F2F2}
\definecolor{gainc}{HTML}{009E73}
\definecolor{lossc}{HTML}{0072B2}
\newcommand{\method}{CounterCredit}
\newcommand{\ind}{\mathbb{I}}

\newtheorem{lemma}{Lemma}
\newtheorem{proposition}{Proposition}
\newtheorem{corollary}{Corollary}
\usepackage{algorithm,algpseudocode}
\graphicspath{{figs/}}
\usepackage{wrapfig}
\usepackage{float}
\makeatletter\g@addto@macro\normalsize{\abovedisplayskip 4pt plus1pt minus2pt\belowdisplayskip\abovedisplayskip\abovedisplayshortskip 2pt\belowdisplayshortskip 2pt}\makeatother

\usepackage{listings}
\usepackage[most]{tcolorbox}
\definecolor{hookgray}{gray}{0.55}\definecolor{ansblue}{HTML}{1F4E9C}\definecolor{grdorange}{HTML}{B85C00}
\definecolor{pbblue}{HTML}{EAF1FB}\definecolor{pbblueT}{HTML}{C5D8F2}
\definecolor{pbgreen}{HTML}{EAF5EC}\definecolor{pbgreenT}{HTML}{C6E3CC}
\definecolor{pborange}{HTML}{FCF1E6}\definecolor{pborangeT}{HTML}{F5D6B8}
\definecolor{cbgray}{HTML}{F5F5F5}\definecolor{cbdark}{HTML}{404040}
\newtcolorbox{promptbox}[3][]{enhanced,breakable,colback=#2,colframe=#3,boxrule=0.6pt,arc=2mm,left=5pt,right=5pt,top=9pt,bottom=4pt,title=#1,fonttitle=\bfseries\small,coltitle=black,attach boxed title to top left={xshift=4mm,yshift=-3mm},boxed title style={colback=#3,colframe=#3,arc=1mm,boxrule=0pt,left=4pt,right=4pt,top=2pt,bottom=2pt}}
\newtcolorbox{casebox}[2][]{enhanced,breakable,colback=cbgray,colframe=cbdark,boxrule=0.8pt,arc=1.5mm,left=5pt,right=5pt,top=5pt,bottom=5pt,title=#2,fonttitle=\bfseries\small,coltitle=white,colbacktitle=cbdark,titlerule=0pt,#1}
\newtcolorbox{summarybox}{enhanced,colback=cbgray,colframe=cbdark,boxrule=0.6pt,arc=1.5mm,left=5pt,right=5pt,top=3pt,bottom=3pt,title=Summary,fonttitle=\bfseries\small,coltitle=white,colbacktitle=cbdark,titlerule=0pt,toptitle=1pt,bottomtitle=1pt,before skip=4pt,after skip=4pt}
\newcommand{\dashsep}{\par\noindent\tikz{\draw[dash pattern=on 2pt off 2pt,gray!60,line width=0.4pt] (0,0) -- (\linewidth,0);}\par}

\iclrfinalcopy

\title{When Should a VLM Look?\\
Paying Only for Visual Calls That Were Needed and Used}
\author{
Kunyu Peng$^{1}$\thanks{Equal contribution.} \quad Junming Liu$^{1,3}$\footnotemark[1] \quad Ruiqi He$^{2}$ \\
\bfseries Qingzhuo Wang$^{1}$ \quad Jianzhong Qi$^{3}$ \quad Xianhui Liu$^{1}$ \\[0.45em]
{\normalfont\normalsize $^{1}$Tongji University \quad $^{2}$Nankai University \quad $^{3}$University of Melbourne} \\[0.3em]
{\normalfont\footnotesize\ttfamily pky@tongji.edu.cn \quad junming.liu.1@student.unimelb.edu.au} \\[0.1em]
{\normalfont\footnotesize\ttfamily heruiqi@mail.nankai.edu.cn \quad 2534123@tongji.edu.cn} \\[0.1em]
{\normalfont\footnotesize\ttfamily jianzhong.qi@unimelb.edu.au \quad lxh@tongji.edu.cn}
}

\begin{document}
\maketitle
\lhead{Preprint.}
\begin{abstract}
Vision-language agents that crop and zoom are trained with rewards that credit a successful tool call, yet a successful call does not show that the model needed to look or used the pixels it received. On our cold-start checkpoint only 10\% to 12\% of visual calls were both needed and used, and released agents make spurious calls 36\% to 87\% of the time on individual benchmarks. Outcome rewards, judge rewards, and branch probes each observe one side of this failure, and about two thirds of what an outcome reward pays goes to calls that were neither needed nor used. \textbf{\method{}} asks both questions of every image-returning call at its realized pre-call state, using the policy's own gold-answer score. A decision value compares the realized visual branch with answering immediately; an evidence value compares the returned crop with random same-size patches substituted into the same call. A call verified on both earns cashback and every other executed call pays rent; the price is bounded so that every correct trajectory outranks every wrong one, and a dual-channel GRPO advantage keeps the price in its own units. From the same cold start, prompt pool, and budget, \textbf{\method{}} reaches 89.5\% on V$^\ast$, 80.2\% on HR-Bench-4K, and 76.4\% on HR-Bench-8K, 6.3 to 9.4 points above outcome-only GRPO at 1.78 against 1.84 calls per question, and lowers the spurious-call rate to 31\% to 36\%, the lowest among the agents evaluated. The same recipe lifts a Qwen3-VL-8B base from 75.4 to 80.8 on average.
\end{abstract}
\begin{figure}[H]
\vspace{-0.6em}
\centering
\includegraphics[width=\linewidth]{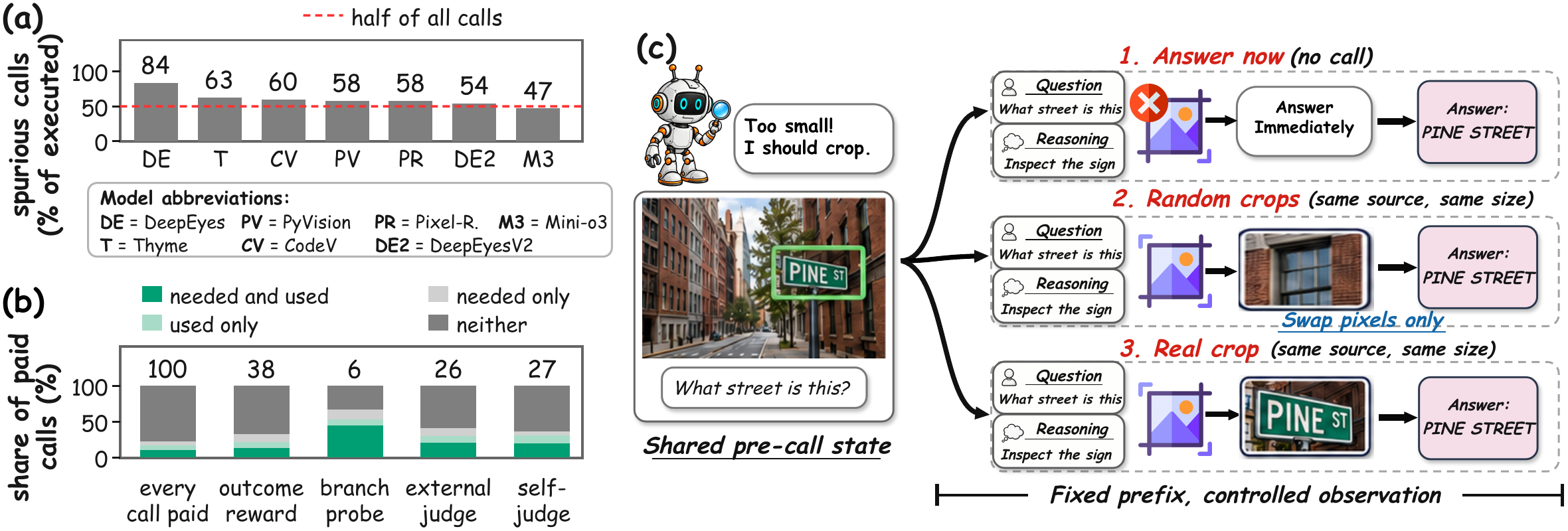}
\vspace{-0.4em}
\caption{\textbf{Overview of spurious visual calls.} (a) Spurious-call rates of released visual agents. (b) Share of calls paid by different reward schemes on the cold-start checkpoint, split by whether they were needed and used; the number above each bar is the calls paid per 100 executed. Both panels average V$^\ast$, HR-Bench-4K, and HR-Bench-8K. (c) Fixed-prefix counterfactuals that separate call necessity from evidence use while holding the realized pre-call state fixed across continuations.}
\label{fig:fake_call}
\vspace{-0.6em}
\end{figure}

\section{Introduction}
\label{sec:intro}

Multimodal models have made remarkable progress in visual understanding, yet VLMs still lag text-only language models as agents that reason through tools and acquire information on their own~\citep{Liu_2025_VisualAgentBench,Wu_2026_VTool-R1,Zhang_2026_MageBench,Zheng_2026_DeepEyes,Sathish_2026_SVRL}. One possible bottleneck is that pretrained visual knowledge is not fully unlocked during post-training~\citep{Xiao_2026_Staying_VIGILant}: a VLM may understand an image when it is given one, yet fail to acquire visual information when it must decide what to inspect~\citep{Zhu_2026_ACTIVE-o3}.
Tool use offers a natural mechanism to bridge this gap: an agent can inspect image regions with a crop tool. Yet trajectory-level success does not reveal whether a visual interaction was useful~\citep{Chen_2026_TUMIX}.

Our key observation is that a successful visual tool call does not show that the decision depended on the look or on the pixels it returned.
Conventional outcome-based evaluation observes only final correctness and therefore cannot distinguish genuinely useful visual interactions from calls that were unnecessary or merely performative~\citep{Hou_2026_CodeV,Chng_2026_SenseSearch,Zhu_2026_ACTIVE-o3,Lai_2026_Mini-o3}.
We refer to calls exhibiting either failure as \emph{spurious}. As illustrated in Figure~\ref{fig:fake_call}, this problem is widespread: averaged over three benchmarks, 47\% to 84\% of the calls made by released visual agents are spurious, and roughly two thirds of the credit assigned by outcome-based rewards goes to calls that are neither needed nor used. Judge-based rewards assess whether an observation appears useful~\citep{Wang_2026_MoCA}, while efficiency penalties merely discourage excess calls~\citep{Wang_2026_AdaTooler-V}; neither establishes that the call was necessary or that its evidence changed the decision. Existing objectives thus conflate \emph{decision value} with \emph{evidence value}, effectively paying agents for looking without verifying that they have seen~\citep{Zhao_2026_Learning}.

We introduce \textbf{\method}, a counterfactual framework for assigning credit to visual interactions. For each visual call, we fix the pre-call prefix and score three continuations, answering without the call, the call with a matched random observation, and the call with the real observation, which separately measure whether the call is necessary and whether its visual evidence affects the decision. We use these quantities to construct a cashback-and-rent reward: a call on a correct trajectory earns cashback only when looking improved the answer and the improvement came through its returned pixels, and every other executed call pays rent. We then optimize this reward with a dual-channel GRPO~\citep{Shao_2024_GRPO} objective that normalizes the outcome and the price in separate channels, so that the price keeps its own scale against the outcome.

\paragraph{Contributions.}
\begin{enumerate}
\item We identify \emph{spurious visual calls} as a possible bottleneck for multimodal agentic reasoning, introduce two counterfactual values that characterize four types of visual interactions, and show that existing outcome-based rewards assign much of their credit to calls with little or no decision benefit, even with  trajectories that reach the correct final answer.

\item We develop \textbf{\method}, a cashback-and-rent credit mechanism with a dual-channel GRPO objective that pays verified evidence use and charges unproductive calls, with an outcome-dominance guarantee that preserves the ordering induced by correctness.

\item \method\ averages 78.5 on the four perception benchmarks, 2.5 points above Mini-o3, the strongest tool agent on these four, and has the lowest spurious-call rate on V$^\ast$, HR-Bench-4K, and HR-Bench-8K, 31\% to 36\%, 9 to 15 points below Mini-o3. It leads every tool agent on all seven benchmarks and remains effective with the Qwen3-VL-8B backbone.
\end{enumerate}

\section{Related Work}

\paragraph{Learning visual tool use.}
V$^\ast$ studies guided visual search~\citep{wu2024vstar}; DeepEyes learns interleaved visual reasoning through RL~\citep{Zheng_2026_DeepEyes}, Pixel Reasoner adds curiosity rewards~\citep{su2025pixelreasoner}, Thyme enables executable image processing~\citep{zhang2025thyme}, and Mini-o3, our cold-start source and baseline, trains long multi-turn search with an outcome-only reward from an LLM judge~\citep{Lai_2026_Mini-o3}. VTool-R1~\citep{Wu_2026_VTool-R1}, ACTIVE-o3~\citep{Zhu_2026_ACTIVE-o3}, and Chain-of-Focus~\citep{zhang2025chainoffocus} train the zoom decision itself, each rewarding the outcome of the trajectory that contains it. Outcome success alone does not establish evidence use.

\paragraph{Rewards for visual actions.}
CodeV's TAPO judges tool-output relevance~\citep{Hou_2026_CodeV}; FaithEyes uses self-judging to score process images~\citep{wang2026faitheyes}. SCCM rewards externally judged sufficiency of the  evidence~\citep{liu2025sccm}, and ViRL shapes step rewards from ground-truth visual rationales~\citep{wang2025virl}. AdaTooler-V scales a tool-use reward component by a per-sample tool benefit score estimated from accuracy with and without the tool~\citep{Wang_2026_AdaTooler-V}. TACO probes answers before and after a tool branch~\citep{feng2026taco}, AXPO resamples the call and its continuation from a fixed thinking prefix and scores the branch by its outcome~\citep{kang2026axpo}, and IGPO rewards sequential belief updates~\citep{wang2026igpo}. TACO and the tool benefit score measure whether the branch helped, and CauAudit measures whether the pixels mattered but only as a post hoc diagnostic, so none measures both at the same realized pre-call state during training.

\paragraph{Credit assignment and causal attribution.}
PACR rewards increases in gold-answer probability along reasoning traces~\citep{yoon2025pacr}; GiGPO estimates step advantages by grouping repeated environment states~\citep{feng2025gigpo}. \citet{ma2026vision} analyze  training effects with a no-tool RL control, finding that training mainly reduces tool-induced harm. Post hoc audits find the same pattern: most questions a tool agent answers need no tool~\citep{guo2026toolbenefit}, a placeholder in place of the returned image leaves aggregate accuracy intact~\citep{shao2026textcall}, and models often ignore the evidence they fetch~\citep{zhang2026insensitivity}. Counterfactual explanation tests~\citep{turpin2023unfaithful} and causal mediation~\citep{pearl2001direct} provide related foundations; CauAudit diagnoses visual-evidence dependence through interventions~\citep{wang2026cauaudit}. \method{} brings such interventions into training and asks at the same pre-call state whether the look was needed and whether the pixels were used.

\section{CounterCredit}
\label{sec:method}

A successful visual call leaves two questions unresolved: \textbf{was the look needed, and were the returned pixels used?} CounterCredit answers both at the realized pre-call state, Figure~\ref{fig:method}. \emph{Decision value} $D_i$ is the gold-answer score gain of the realized branch over answering at once, the total value of a look; \emph{evidence value} $E_i$ is its gain over random patches substituted into the same call, the part carried by the pixels. A call verified on both earns cashback; every other executed call pays rent.

\subsection{Setting and the counterfactual probe}
\label{sec:setting}
Let $x=(I,y)$ denote an example with image $I$ and gold answer $y$; the prompt includes the question and, for multiple-choice questions, the options the benchmark supplies, given verbatim in Appendix~\ref{app:prompts}. Policy $\pi_\theta$ generates a trajectory $\tau$ interleaving text with $n(\tau)$ executed visual calls $c_1,\dots,c_{n(\tau)}$. The policy has one visual tool, a crop of the image. Each call $c_i$ returns an image crop $o_i$, and the final answer receives outcome reward $r_{\mathrm{out}}(\tau)\in\{0,1\}$.

Let $h_i$ be the exact prefix immediately before $c_i$, including the prompt, previous interactions, and generated reasoning; Appendix~\ref{app:prompts} gives the turn structure and Appendix~\ref{app:case} four trajectories. Each probed call thus separates the pre-call context $h_i$, the call text $c_i$, and the returned pixels $o_i$; all three continuations below retain $h_i$.

\begin{figure}[t]
\centering
\includegraphics[width=\linewidth]{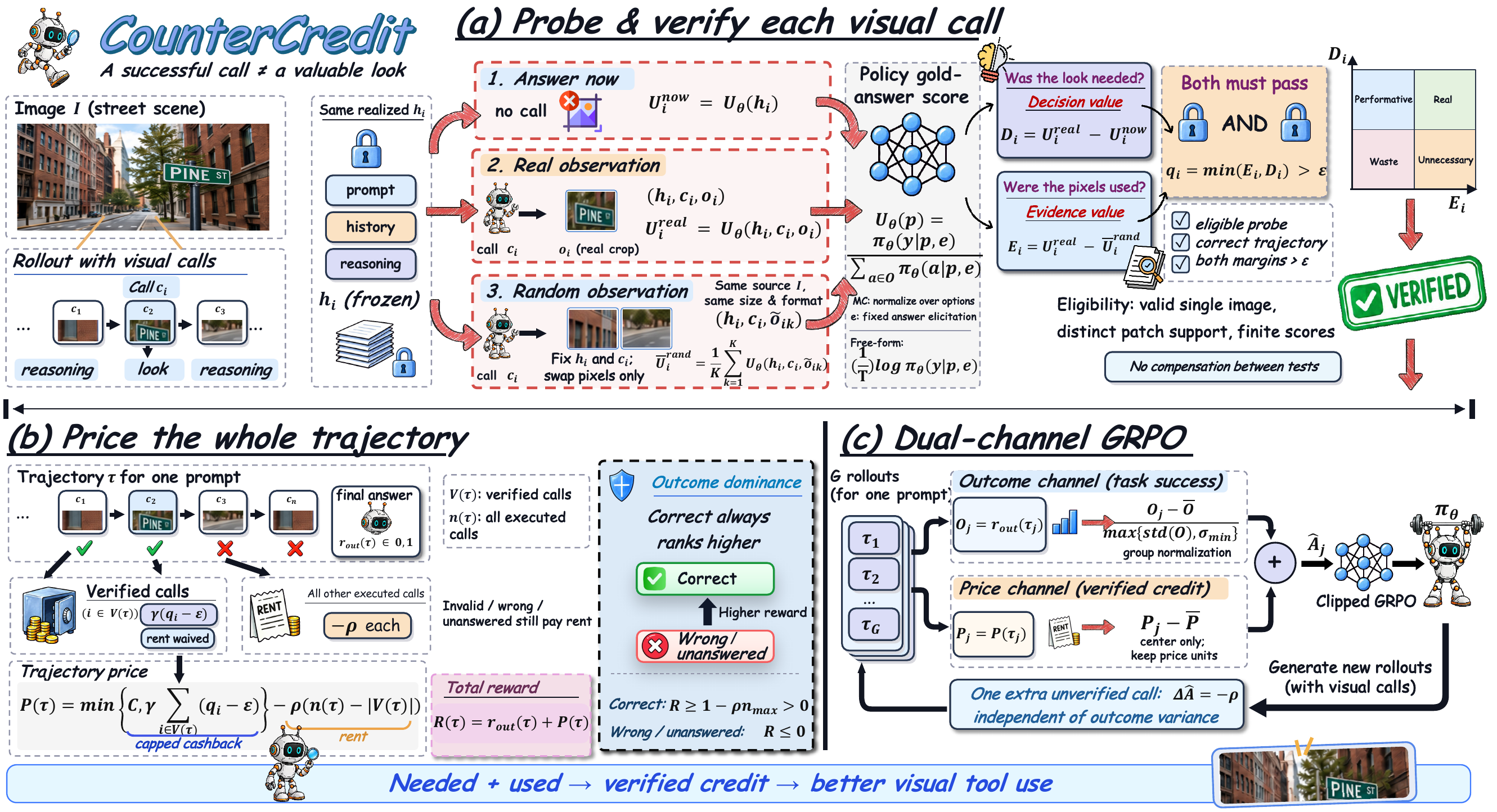}
\caption{\textbf{\method{} overview}. \textbf{(a)} Each visual call is probed from its realized pre-call state with three continuations, \textbf{\ding{182} answer now, \ding{183} real observation, \ding{184} random observation}, $D_i$ asks whether the look was needed, $E_i$ whether the pixels were used, and a call is verified when both exceed $\varepsilon$ on a correct trajectory. \textbf{(b)} The price adds capped cashback on verified calls and rent on other executed calls. \textbf{(c)} Dual-channel GRPO normalizes the outcome channel and only centers the price channel.}
\label{fig:method}
\end{figure}

All counterfactuals begin from the same realized $h_i$. The tool is a crop, and only a crop, because the evidence value needs a matched null return: a crop returns only pixels from $I$, so random same-size patches from $I$ serve as that reference, while a tool that returns anything else has no such reference yet. We compare three continuations of $h_i$, shown in Figure~\ref{fig:method}(a): the answer-now continuation elicits an answer directly from $h_i$; the real-observation continuation $(h_i,c_i,o_i)$ appends the call the policy actually issued and the crop it received; and each of $K$ random-observation continuations $(h_i,c_i,\tilde{o}_{ik})$ appends the same $c_i$ and an unscreened same-size patch sampled uniformly at random from the original image and rendered in the same format as the real observation.

For each of these prefixes $p$, define the policy's gold-answer score as
\begin{equation}
U_\theta(p)=\frac{\pi_\theta(y\mid p,e)}{\sum_{a\in\mathcal{O}}\pi_\theta(a\mid p,e)},
\label{eq:U}
\end{equation}
where $e$ is a fixed answer elicitation appended as a user turn, given verbatim in Appendix~\ref{app:prompts}, and $\mathcal{O}$ is the set of option letters of the question, so the score renormalizes the gold-option probability over them. For free-form tasks, we instead use the mean log-probability per token of the reference answer $y$ under teacher forcing, and its differences are gated by a deadzone calibrated separately on free-form calls, Appendix~\ref{app:config}. Probe scores are evaluated under the current policy with gradients stopped.

\subsection{Was the look needed? The decision value}
\label{sec:decision}
Turning to the first question, \emph{was the look needed?} A correct answer after a call does not establish that the call was needed. The answer may already be available in $h_i$ from the full image or the reasoning written before the call. To answer the first question, we compare the gold-answer score of the realized visual branch with that of answering immediately from the same state. The decision value is
\begin{equation}
U^{\mathrm{real}}_i=U_\theta(h_i,c_i,o_i),\qquad U^{\mathrm{now}}_i=U_\theta(h_i),\qquad D_i:=U^{\mathrm{real}}_i-U^{\mathrm{now}}_i.
\label{eq:decision}
\end{equation}
A positive $D_i$ means that the realized visual branch increased the gold-answer score relative to answering from $h_i$. Any additive score contribution that depends only on the shared pre-call prefix cancels in $D_i$, as formalized in Lemma~\ref{lem:shared} of Appendix~\ref{app:proofs}.

\subsection{Was the look used? The evidence value}
\label{sec:values}
\begin{wrapfigure}{r}{0.36\linewidth}
\vspace{-1.2em}
\centering
\includegraphics[width=\linewidth]{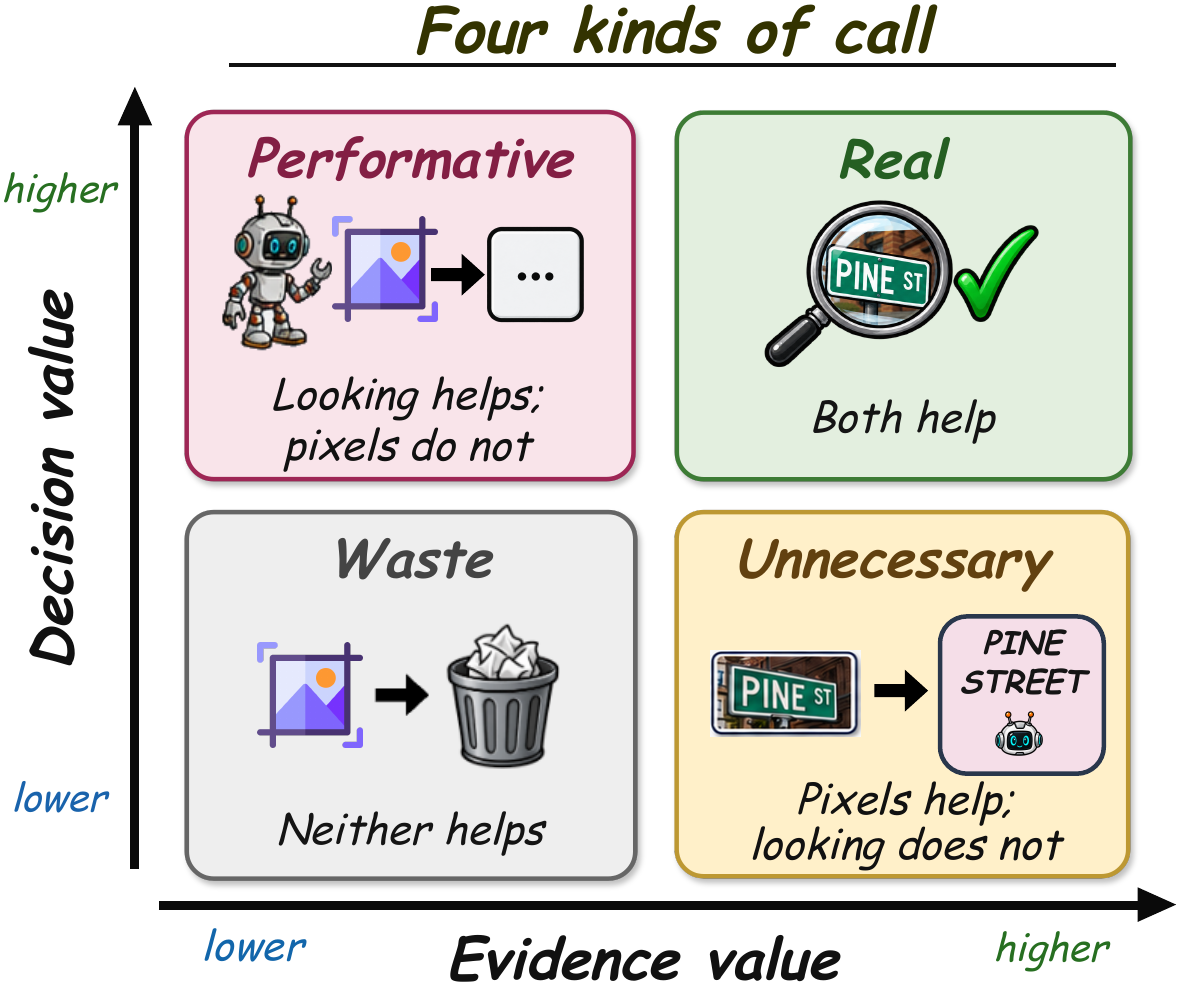}
\caption{\textbf{The four call regions} by evidence value $E_i$ and decision value $D_i$ for each visual call.}
\label{fig:audit}
\vspace{-0.6em}
\end{wrapfigure}
Turning to the second question, \emph{were the returned pixels used?} A useful branch does not imply that its returned pixels caused the improvement. The call itself, or simply receiving an observation, may affect the gold-answer score even when the visual content is uninformative. To answer the second question, we hold the prefix $h_i$ and the call $c_i$, including its text arguments, fixed and replace only the returned pixels with random patches. The evidence value is
\begin{equation}
\bar{U}^{\mathrm{rand}}_i=\frac{1}{K}\sum_{k=1}^{K}U_\theta(h_i,c_i,\tilde{o}_{ik}),\quad E_i:=U^{\mathrm{real}}_i-\bar{U}^{\mathrm{rand}}_i.
\label{eq:evidence}
\end{equation}
$E_i$ measures the gain in gold-answer score from the returned pixels relative to the mean score of the random-observation continuations. $D_i$ compares the real crop with answering now, and $E_i$ compares it with random patches of the same image, so $D_i>0$ says the look improved the answer and $E_i>0$ says the improvement came from the returned pixels. Needed and used denote score gains beyond the deadzone over answering now and over random patches, including added confidence in answers the policy would already get right. Both conditions must hold, as formalized in Section~\ref{sec:billing}. Table~\ref{tab:probe-sens} in Appendix ~\ref{sec:table13} reports used-label agreement with the default across reference sample counts, patch locations, and reference types.

\subsection{Both must pass: verification and pricing}
\label{sec:billing}
The pair $(E_i,D_i)$ induces the credit geometry of Figure~\ref{fig:audit}. Calls with $E_i>0$ and $D_i>0$ are \emph{real}: looking improved the answer and the pixels mattered. Calls with $D_i>0$ but $E_i\leq0$ are \emph{performative}: random patches would have helped as much. Calls with $E_i>0$ but $D_i\leq0$ are \emph{unnecessary}: the pixels are informative, but the policy could already answer. Calls with neither effect are \emph{waste}. Both failure regions are populated on the cold-start checkpoint, Figure~\ref{fig:fake_call}(b), and CounterCredit pays only the real region, with credit limited by both references. Proofs of the propositions below are in Appendix~\ref{app:proofs}, together with the assumptions and scope of each result.

Paying a call requires one number that summarizes both effects, so we define its credit as
\begin{equation}
q_i=\min\{E_i,D_i\}.
\label{eq:q}
\end{equation}
The following proposition justifies the minimum.
\begin{proposition}[Robust advantage over two references]
\label{prop:robust}
For $\lambda\in[0,1]$ let $B(\lambda)=\lambda U^{\mathrm{now}}_i+(1-\lambda)\bar{U}^{\mathrm{rand}}_i$, the expected score of a reference that answers at once with probability $\lambda$ and otherwise makes the call with a random patch. Then $\inf_{\lambda}[U^{\mathrm{real}}_i-B(\lambda)]=U^{\mathrm{real}}_i-\max\{U^{\mathrm{now}}_i,\bar{U}^{\mathrm{rand}}_i\}=q_i$, and $q_i=\sup\{m:U^{\mathrm{real}}_i\geq U^{\mathrm{now}}_i+m,\ U^{\mathrm{real}}_i\geq\bar{U}^{\mathrm{rand}}_i+m\}$.
\end{proposition}
The minimum is therefore the advantage that survives the least favourable weighting of the two references. Verification requires this advantage to exceed a deadzone $\varepsilon$, fixed before RL by a null calibration on contrasts built from random patches alone and held fixed during training; Appendix~\ref{app:config} gives the procedure, the false-verification rate at each grid value, and the separate values used for multiple-choice and free-form calls. The \emph{spurious} calls $\{i:q_i\leq\varepsilon\}$ comprise the performative, unnecessary, and waste regions of Figure~\ref{fig:audit} together with the band of the real region within $\varepsilon$ of either axis. Let $a_i\in\{0,1\}$ indicate probe eligibility, requiring a valid single-image return, distinct random-patch support, and finite probe scores. Call $i$ is verified when
\begin{equation}
v_i=r_{\mathrm{out}}(\tau)\,a_i\,\ind[q_i>\varepsilon].
\label{eq:verify}
\end{equation}
Verification alone does not discourage indiscriminate tool use, so every unverified executed call pays rent $\rho$, including failed parses and calls on wrong or unanswered trajectories. Verified calls are exempt and receive cashback proportional to their margin above the deadzone, capped at $C$ per trajectory, Figure~\ref{fig:method}(b). With $\mathcal{V}(\tau)=\{i:v_i=1\}$,
\begin{equation}
P(\tau)=\min\Big\{C,\ \gamma\sum_{i\in\mathcal{V}(\tau)}(q_i-\varepsilon)\Big\}-\rho\big(n(\tau)-|\mathcal{V}(\tau)|\big),\qquad R(\tau)=r_{\mathrm{out}}(\tau)+P(\tau).
\label{eq:reward}
\end{equation}
For an eligible call on a correct trajectory, crossing $\varepsilon$ removes the rent, so payment jumps by $\rho$ at the deadzone and then grows with $q_i$ until the cap binds. The price is added to the outcome reward, so it could in principle rank a cheap wrong trajectory above an expensive correct one. The following proposition rules this out.
\begin{proposition}[Outcome dominance]
\label{prop:dominance}
Suppose every trajectory executes at most $n_{\max}$ calls and $\rho\,n_{\max}<1$. A correct trajectory has $R(\tau)\geq1-\rho\,n_{\max}$, whereas a wrong or unanswered trajectory has $R(\tau)=-\rho\,n(\tau)\leq0$, so every correct trajectory outranks any wrong or unanswered one by at least $1-\rho\,n_{\max}$, regardless of their respective call counts.
\end{proposition}
The bound holds because cashback is nonnegative and restricted to correct trajectories while total rent is at most $\rho\,n_{\max}$; with $\rho=0.05$ and $n_{\max}=6$ the margin is $0.7$, and $C=0.7$ keeps correct-trajectory rewards in $[0.7,1.7]$; the guarantee weakens linearly in the call budget and vanishes at $\rho\,n_{\max}=1$. Rent also covers verification noise. With $K=3$ the random-patch mean fluctuates, so $q_i$ is a noisy estimate of the credit $q^\star_i$ that the same call would receive against the expected patch score, and a crop with $q^\star_i\leq\varepsilon$ can still clear the deadzone by chance. The following proposition shows that the rent makes such luck unprofitable.
\begin{proposition}[Chance verification does not pay]
\label{prop:null}
For an executed call let $A_i=\{q^\star_i\leq\varepsilon\}$ be the event that its value at the reference mean falls short of the verification margin, and suppose $\Pr(A_i)>0$ and $\Pr(v_i=1\mid A_i)\leq\alpha\in[0,1)$. Its uncapped price satisfies $\mathbb{E}[X_i\mid A_i]\leq\alpha\gamma(1-\varepsilon)-(1-\alpha)\rho$, which is negative whenever the rent satisfies $\rho>\frac{\alpha}{1-\alpha}\gamma(1-\varepsilon)$.
\end{proposition}
The rent thus offsets in expectation what chance verification collects, at our constants for every $\alpha$ below $9.5\%$; for free-form calls the bound assumes a score range of at most $B$, with $\gamma(B-\varepsilon)$ in place of $\gamma(1-\varepsilon)$. The null calibration of Appendix~\ref{app:config}, run before training on training-pool questions and reproduced on V$^\ast$ and HR-Bench-4K, gives proxy rates of $2.4\%$ and $2.5\%$ for it; Appendix~\ref{app:proofs} gives the proof and shows that the cap only lowers the trajectory price.

\subsection{Dual-channel GRPO optimization}
\label{sec:optimization}
Standard GRPO~\citep{Shao_2024_GRPO} normalizes total reward within each prompt group $g(j)$, which makes the effective scale of a fixed price depend on group composition: outcome variation suppresses price differences and identical outcomes amplify them. We therefore separate rollout $j$'s reward into outcome $O_j=r_{\mathrm{out}}(\tau_j)$ and price $P_j=P(\tau_j)$, and use
\begin{equation}
\widehat{A}_j=\frac{O_j-\overline{O}_{g(j)}}{\max\{\mathrm{std}(O_{g(j)}),\sigma_{\min}\}}+\big(P_j-\overline{P}_{g(j)}\big),
\label{eq:dual-channel}
\end{equation}
with $\sigma_{\min}=0.15$. The outcome channel retains GRPO's group-relative normalization, whereas the price channel is only centered and therefore remains in its original units, Figure~\ref{fig:method}(c). \begin{proposition}[Price scale under normalization]
\label{prop:scale}
In a group whose rollouts are all correct, carry no verified call, and differ in their unverified call counts $n_j$, single-channel normalization gives $A_j=-(n_j-\bar{n})/\mathrm{std}(n)$, independent of $\rho$, whereas Equation~\ref{eq:dual-channel} gives $\widehat{A}_j=-\rho\,(n_j-\bar{n})$.
\end{proposition}
Standard GRPO thus cannot see the rent's magnitude in such groups, while the price channel keeps it: an otherwise identical rollout with one additional unverified call differs in advantage by exactly $\rho$. When no rollout is correct the price advantage is set to zero, since prices there carry only call counts. In a group holding both outcomes, Proposition~\ref{prop:dominance} carries over: a correct rollout's $\widehat{A}_j$ exceeds a wrong one's by at least $1/\max\{\mathrm{std}(O),\sigma_{\min}\}-\rho\,n_{\max}$, which is at least $1.5$ since the sample standard deviation of eight binary outcomes is at most $0.54$.

We substitute $\widehat{A}_j$ into the clipped GRPO objective without backpropagating through probe scores; a correct trajectory with $m$ eligible calls costs $m(K+2)$ score evaluations, Algorithm~\ref{alg:cc} in Appendix~\ref{app:config}.

\section{Experiments}
\label{sec:experiments} 

\subsection{Setup}
\label{sec:setup}

\paragraph{Training.}
Qwen2.5-VL-7B-Instruct~\citep{bai2025qwen25vl} is fine-tuned on 6,960 filtered Mini-o3 trajectories~\citep{Lai_2026_Mini-o3}, then trained with GRPO under the reward of Section~\ref{sec:method} on 16,000 VisualProbe and DeepEyes prompts~\citep{Lai_2026_Mini-o3,Zheng_2026_DeepEyes}, deduplicated, decontaminated against every evaluation image, and filtered by difficulty. The outcome reward $r_{\mathrm{out}}$ counts an answer correct when it matches the reference exactly or by option text; free-form answers that fail both are judged by GPT-4o~\citep{openai2024gpt4o}, whose judgments agree with human labels on 99.4\% of a 500-answer sample, Appendix~\ref{app:protocol}. Appendix~\ref{app:data} details curation and Appendix~\ref{app:config} the configuration.

\paragraph{Benchmarks and metrics.}
Evaluation covers four perception benchmarks, V$^\ast$~\citep{wu2024vstar}, HR-Bench-4K and HR-Bench-8K~\citep{wang2024hrbench}, and MME-RealWorld~\citep{zhang2024mmerealworld}, and three general benchmarks, MMStar~\citep{chen2024mmstar}, ChartQA~\citep{masry2022chartqa}, and BLINK~\citep{fu2024blink}. On V$^\ast$ and both HR-Bench splits, the \emph{spurious-call rate} is the share of executed image-returning calls with $\min\{E_i,D_i\}\leq\varepsilon$, with each agent's $E_i$ and $D_i$ computed under its own gold-answer score on its own prefixes. Appendix~\ref{app:protocol} gives the evaluation details.

\paragraph{Baselines.}
Baselines are seven open tool-using agents re-evaluated from their released checkpoints in the same pipeline, DeepEyes, Pixel-Reasoner, Mini-o3, Thyme, CodeV-RL, DeepEyesV2~\citep{hong2025deepeyesv2}, and PyVision-RL~\citep{zhao2026pyvisionrl}; three models without visual tools, Qwen2.5-VL-7B, Qwen3-VL-8B~\citep{bai2025qwen3vl}, and InternVL3-8B~\citep{zhu2025internvl3}; and GPT-5~\citep{openai2025gpt5} and Gemini-2.5-Pro~\citep{google2025gemini25} through their APIs. Appendix~\ref{app:protocol} gives the audit rule for code-based agents. TACO has no public checkpoint; Figure~\ref{fig:fake_call}(b) audits its branch probe and Table~\ref{tab:ablation} retrains its branch reward.

\subsection{Main Results}
\label{sec:main-results}

\begin{table}[t]
\caption{Accuracy (\%) on four perception and three general benchmarks, all evaluated under our protocol; Avg. is the mean over the seven. Bold: best per column among tool-using agents; underline: second best. Cell colour: gain over Qwen2.5-VL-7B. The ours row is the run used throughout; $\pm$ is the sample standard deviation over three seeds, Table~\ref{tab:seeds} (Appendix ~\ref{sec:seed variation}). Baselines are single runs, and their spread is not measured.}
\label{tab:main}
\centering\footnotesize
\setlength{\tabcolsep}{3pt}
\begin{tabular}{lcccccccc}
\toprule
 & \multicolumn{4}{c}{\textbf{Perception}} & \multicolumn{3}{c}{\textbf{General}} &  \\
\cmidrule(lr){2-5}\cmidrule(lr){6-8}
\textbf{Model} & \textbf{V$^\ast$} & \textbf{HR-4K} & \textbf{HR-8K} & \textbf{MME-RW} & \textbf{MMStar} & \textbf{ChartQA} & \textbf{BLINK} & \textbf{Avg.} \\
\midrule
\rowcolor{blockbg}\multicolumn{9}{l}{\textbf{\textit{Closed models (API)}}} \\
GPT-5 & 72.3 & 75.6 & 73.8 & 68.9 & 73.4 & 76.5 & 69.7 & 72.9 \\
Gemini-2.5-Pro & 79.1 & 84.2 & 81.1 & 58.7 & 73.4 & 83.8 & 73.6 & 76.3 \\
\rowcolor{blockbg}\multicolumn{9}{l}{\textbf{\textit{Open models without a visual tool}}} \\
Qwen2.5-VL-7B & 77.0 & 69.1 & 65.6 & 58.0 & 64.4 & 83.9 & 56.2 & 67.7 \\
Qwen3-VL-8B & 85.9 & 79.2 & 74.3 & 60.7 & 70.5 & 88.4 & 68.9 & 75.4 \\
InternVL3-8B & 70.2 & 70.4 & 69.1 & 61.7 & 68.2 & 86.0 & 55.6 & 68.7 \\
\rowcolor{blockbg}\multicolumn{9}{l}{\textbf{\textit{Tool-using agents (7B)}}} \\
DeepEyes & \cellcolor{gainc!39!white}84.8 & \cellcolor{gainc!22!white}73.4 & \cellcolor{gainc!27!white}70.9 & \cellcolor{gainc!31!white}64.3 & \cellcolor{gainc!3!white}65.1 & 82.6 & \cellcolor{gainc!6!white}57.4 & \cellcolor{gainc!18!white}71.2 \\
Pixel-Reasoner & \cellcolor{gainc!34!white}83.8 & \cellcolor{gainc!8!white}70.6 & \cellcolor{gainc!8!white}67.2 & \cellcolor{gainc!32!white}64.5 & \cellcolor{gainc!2!white}64.8 & 81.9 & \cellcolor{gainc!3!white}56.9 & \cellcolor{gainc!11!white}70.0 \\
Mini-o3 & \cellcolor{gainc!52!white}\underline{87.4} & \cellcolor{gainc!45!white}\underline{78.0} & \cellcolor{gainc!39!white}73.4 & \cellcolor{gainc!36!white}\underline{65.2} & 64.1 & \cellcolor{gainc!7!white}85.3 & 55.6 & \cellcolor{gainc!25!white}72.7 \\
Thyme & \cellcolor{gainc!34!white}83.8 & \cellcolor{gainc!39!white}76.8 & \cellcolor{gainc!29!white}71.3 & \cellcolor{gainc!34!white}64.7 & \cellcolor{gainc!4!white}65.2 & \cellcolor{gainc!10!white}85.9 & \cellcolor{gainc!1!white}56.4 & \cellcolor{gainc!21!white}72.0 \\
CodeV-RL & \cellcolor{gainc!39!white}84.8 & \cellcolor{gainc!36!white}76.3 & \cellcolor{gainc!30!white}71.5 & \cellcolor{gainc!33!white}64.6 & \cellcolor{gainc!14!white}\underline{67.2} & \cellcolor{gainc!1!white}84.1 & \cellcolor{gainc!10!white}\underline{58.3} & \cellcolor{gainc!24!white}72.4 \\
DeepEyesV2 & \cellcolor{gainc!31!white}83.2 & \cellcolor{gainc!42!white}77.6 & \cellcolor{gainc!38!white}73.2 & \cellcolor{gainc!35!white}64.9 & \cellcolor{gainc!5!white}65.4 & \cellcolor{gainc!12!white}\underline{86.3} & \cellcolor{gainc!8!white}57.8 & \cellcolor{gainc!24!white}72.6 \\
PyVision-RL & \cellcolor{gainc!49!white}86.9 & \cellcolor{gainc!41!white}77.3 & \cellcolor{gainc!41!white}\underline{73.8} & \cellcolor{gainc!22!white}62.4 & \cellcolor{gainc!6!white}65.7 & \cellcolor{gainc!11!white}86.2 & \cellcolor{gainc!8!white}57.9 & \cellcolor{gainc!26!white}\underline{72.9} \\
\midrule
CounterCredit (ours) & \cellcolor{gainc!62!white}{\scriptsize\textbf{89.5}$\pm$0.81} & \cellcolor{gainc!56!white}{\scriptsize\textbf{80.2}$\pm$0.75} & \cellcolor{gainc!54!white}{\scriptsize\textbf{76.4}$\pm$0.90} & \cellcolor{gainc!49!white}{\scriptsize\textbf{67.7}$\pm$0.50} & \cellcolor{gainc!22!white}{\scriptsize\textbf{68.9}$\pm$0.35} & \cellcolor{gainc!16!white}{\scriptsize\textbf{87.1}$\pm$0.30} & \cellcolor{gainc!33!white}{\scriptsize\textbf{62.8}$\pm$0.65} & \cellcolor{gainc!42!white}{\scriptsize\textbf{76.1}$\pm$0.40} \\
\bottomrule
\end{tabular}
\end{table}
\textbf{The lead spans all seven benchmarks.} \method{} leads every tool agent on all seven benchmarks in Table~\ref{tab:main}, although the agents differ in tools, prompts, and interaction formats, so the matched comparisons carry the causal claim. On V$^\ast$, HR-Bench-4K, and HR-Bench-8K it adds 12.5, 11.1, and 10.8 points over the base and exceeds the best tool agent by 2.1, 2.2, and 2.6. The three general benchmarks rise as well, by 3.2 to 6.6 points, the largest gains among the tool agents, and the three resolution scores exceed GPT-5, which has no tool loop.

\textbf{\method{} pairs its accuracy with selective tool use.} Table~\ref{tab:calls} places its call count in the middle of the field, 1.76 to 1.87 per question, with the lowest spurious-call rate, 31.4\% to 36.3\%, each judged under its own gold-answer score; re-scoring the same V$^\ast$ rollouts with one common scorer, the SFT cold start, and with an outside scorer that took no part in training leaves \method{} lowest under both and reorders the seven baselines by at most one adjacent swap, Spearman $\rho=0.96$ and $0.75$ against their own scorers, Appendix~\ref{app:results}.

\textbf{Mini-o3 is the matched comparison.} It uses the same crop tool and cold-start data under an outcome reward from an LLM judge: on V$^\ast$ \method{} scores 2.1 points higher, makes 1.78 calls per question against 2.50, and lowers the spurious-call rate from 40.7\% to 31.4\%. Averaged over the three resolution benchmarks, Figure~\ref{fig:tradeoff} places \method{} beyond the observed baseline frontier, at 82.0 points in 8.1 seconds against 74 to 78 points for the agents under six seconds and 79.6 for the best baseline, which takes 10.9.

\begin{table}[t]
\begin{minipage}[t]{0.66\linewidth}
\caption{Calls per question, latency in seconds per question at batch size 1 on one NVIDIA H20 with vLLM, and spurious-call rate, the share of executed image-returning calls with $\min(E,D)\leq\varepsilon$. Bold: lowest; underline: second lowest.}
\label{tab:calls}
\centering\scriptsize
\setlength{\tabcolsep}{1.6pt}
\begin{tabular}{lccccccccc}
\toprule
 & \multicolumn{3}{c}{\textbf{Calls / question}} & \multicolumn{3}{c}{\textbf{Latency (s)}} & \multicolumn{3}{c}{\textbf{Spurious-call rate (\%)}} \\
\cmidrule(lr){2-4}\cmidrule(lr){5-7}\cmidrule(lr){8-10}
\textbf{Model} & \textbf{V$^\ast$} & \textbf{HR-4K} & \textbf{HR-8K} & \textbf{V$^\ast$} & \textbf{HR-4K} & \textbf{HR-8K} & \textbf{V$^\ast$} & \textbf{HR-4K} & \textbf{HR-8K} \\
\midrule
DeepEyes & 0.91 & 0.98 & 0.96 & 16.1 & 14.9 & 15.8 & 87.4 & 83.6 & 82.1 \\
Pixel-Reasoner & 0.95 & 0.84 & 0.82 & 2.0 & 2.0 & 3.0 & 68.1 & 52.5 & 52.1 \\
Mini-o3 & 2.50 & 2.46 & 2.60 & 7.3 & 12.1 & 13.4 & 40.7 & \underline{50.5} & \underline{51.0} \\
Thyme & 0.14 & 0.19 & 0.22 & 2.3 & 2.6 & 2.9 & \underline{36.1} & 77.2 & 75.3 \\
CodeV-RL & 1.84 & 0.43 & 0.47 & 6.7 & 3.3 & 3.6 & 52.4 & 73.7 & 53.9 \\
DeepEyesV2 & 1.23 & 1.27 & 1.32 & 4.4 & 5.7 & 6.3 & 51.3 & 55.3 & 55.3 \\
PyVision-RL & 2.87 & 3.30 & 3.60 & 14.3 & 20.0 & 23.0 & 54.6 & 56.4 & 63.9 \\
\midrule
\rowcolor{oursbg}CounterCredit (ours) & 1.78 & 1.76 & 1.87 & 5.4 & 8.8 & 10.0 & \textbf{31.4} & \textbf{35.1} & \textbf{36.3} \\
\bottomrule
\end{tabular}
\end{minipage}\hfill
\begin{minipage}[t]{0.32\linewidth}
\vspace{0pt}
\makeatletter\def\@captype{figure}\makeatother
\centering
\includegraphics[width=\linewidth]{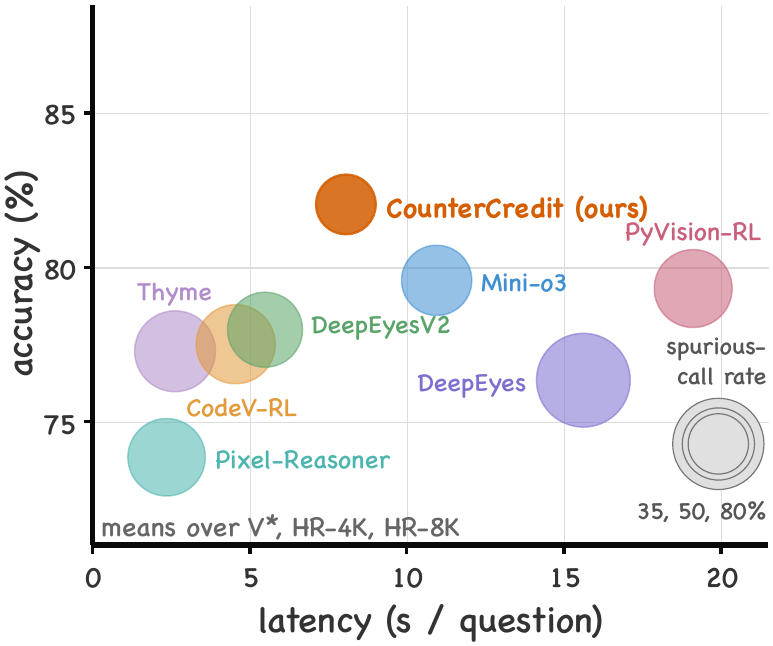}
\caption{Accuracy against latency; marker area is the spurious-call rate.}
\label{fig:tradeoff}
\end{minipage}
\end{table} 

\subsection{Auditing Existing Rewards}
\label{sec:efficiency}

\textbf{Most calls pass neither test.} Figure~\ref{fig:fake_call}(b) audits four rewards on the SFT checkpoint's trajectories: the outcome reward, which credits every call on a correct trajectory; the branch probe of TACO~\citep{feng2026taco}; the external crop judge of CodeV~\citep{Hou_2026_CodeV}, served by GPT-4o~\citep{openai2024gpt4o}; and the self-judge of FaithEyes~\citep{wang2026faitheyes}, answered by the policy itself. Base rates are low: 14\% to 18\% of executed calls pass each test alone and 10\% to 12\% pass both; Appendix~\ref{app:results} gives each payment rule and the numbers per benchmark.

\textbf{The two tests are not one test asked twice.} Only 57\% to 85\% of needed calls were also used and 54\% to 78\% of used calls were needed, so a reward that checks one test alone pays 15\% to 46\% of its verified calls for the wrong reason. Re-decoding the same calls with a random patch in place of the real crop separates them at the answer level too: real calls change the decoded answer sixteen times as often as waste, and performative calls change it as rarely as waste although their decision value is positive, Table~\ref{tab:answerlevel} (Appendix ~\ref{sec:F1}).

\textbf{None of the four rewards asks both questions.} On V$^\ast$ the outcome reward's payments pass both tests 12.8\% of the time against a base rate of 11.8\%, and the two judges, which read the crop, reach 21.0\% and 14.7\%, with HR-Bench in Table~\ref{tab:audit-full}  (Appendix ~\ref{sec:F1}). None of the three changes the call it scores, and 56\% to 77\% of their payments pass neither test. The branch probe of TACO does change it, and one intervention buys precision: 42\% to 50\% of its payments pass both tests, at 5 to 7 per hundred. It removes the whole branch, so it cannot see whether the pixels mattered: 19.4\% of its payments on V$^\ast$ go to performative calls, the region the second question closes.

\subsection{Ablation}
\label{sec:ablation}

\begin{table}[t]
\caption{\textbf{Training stages and reward ablation. }All RL rows start from the SFT checkpoint with the same prompt pool and budget and differ only in the reward; $\gamma=0$ removes the cashback, $\rho=0$ the rent, and the other arms are defined in the text. Avg.: mean of the seven benchmarks; calls and spurious-call rate on V$^\ast$. The \method{} row is the seed shared with every arm; its spread over three seeds is in Tables~\ref{tab:main} and \ref{tab:seeds}.}
\label{tab:ablation}
\centering\scriptsize
\setlength{\tabcolsep}{2.4pt}
\begin{tabular}{lcccccccccc}
\toprule
 & \multicolumn{4}{c}{\textbf{Perception}} & \multicolumn{3}{c}{\textbf{General}} &  & \multicolumn{2}{c}{\textbf{V$^\ast$ calls}} \\
\cmidrule(lr){2-5}\cmidrule(lr){6-8}\cmidrule(lr){10-11}
\textbf{Model} & \textbf{V$^\ast$} & \textbf{HR-4K} & \textbf{HR-8K} & \textbf{MME-RW} & \textbf{MMStar} & \textbf{ChartQA} & \textbf{BLINK} & \textbf{Avg.} & \textbf{Calls} & \textbf{Spur.\,\%} \\
\midrule
\rowcolor{blockbg}\multicolumn{11}{l}{\textbf{\textit{Starting points}}} \\
Qwen2.5-VL-7B & 77.0 & 69.1 & 65.6 & 58.0 & 64.4 & 83.9 & 56.2 & 67.7 & 0 & -- \\
SFT (cold start) & 71.7 & 70.4 & 61.6 & 55.6 & 60.3 & 80.8 & 54.5 & 65.0 & 2.68 & 88.2 \\
\rowcolor{blockbg}\multicolumn{11}{l}{\textbf{\textit{Which calls are verified}}} \\
Outcome-only GRPO & 80.1 & 73.9 & 69.4 & 56.8 & 61.7 & 86.4 & 54.3 & 68.9 & 1.84 & 82.1 \\
Branch reward (TACO) & 85.1 & 76.2 & 72.0 & 60.7 & 62.8 & 86.9 & 57.4 & 71.6 & 2.45 & 62.9 \\
Verify on $E$ only & 82.2 & 74.6 & 68.3 & 59.0 & 62.3 & 86.4 & 53.1 & 69.4 & 1.46 & 71.3 \\
Verify on $D$ only & 84.3 & 76.5 & 72.5 & 60.5 & 63.5 & 86.5 & 55.5 & 71.3 & 1.65 & 65.0 \\
\rowcolor{blockbg}\multicolumn{11}{l}{\textbf{\textit{What verification buys}}} \\
CounterCredit, $\gamma=0$ & 81.2 & 74.6 & 70.2 & 59.4 & 62.9 & 86.3 & 55.4 & 70.0 & 1.35 & 62.2 \\
CounterCredit, $\rho=0$ & 85.9 & 77.9 & 74.9 & 63.0 & 65.3 & 86.9 & 58.0 & 73.1 & 2.70 & 48.7 \\
\rowcolor{blockbg}\multicolumn{11}{l}{\textbf{\textit{How the price enters training}}} \\
Single-channel GRPO & 62.8 & 68.8 & 63.9 & 57.1 & 60.0 & 83.7 & 55.3 & 64.5 & 0.78 & 62.5 \\
\midrule
\rowcolor{oursbg}CounterCredit & 89.5 & 80.2 & 76.4 & 67.7 & 68.9 & 87.1 & 62.8 & 76.1 & 1.78 & 31.4 \\
\bottomrule
\end{tabular}
\end{table}

\textbf{Accuracy depends on which calls survive, not on how many.} The SFT checkpoint calls 2.68 times per question on V$^\ast$, 88.2\% of those calls are spurious, and its average falls 2.7 points below the base, Table~\ref{tab:ablation}. Outcome-only GRPO cuts the calls by a third and gains 1.2 points over the base, yet reduces the spurious share only to 82.1\%: a reward that pays every call on a correct trajectory equally changes how often the policy calls, not which calls survive. \method{} averages 76.1 against 68.9 at 1.78 calls per question against 1.84, and the spurious share falls to 31.4\%. Every RL row shares the checkpoint, prompt pool, budget, decoding, and seed, so this comparison carries the paper's causal claim; outcome-only GRPO uses our cold start and six-call budget, whereas the released Mini-o3 of Table~\ref{tab:main} uses its own.

\textbf{No component of the reward is redundant.} Verifying on $E_i$ or on $D_i$ alone recovers 0.5 and 2.4 of the 7.2-point gain and leaves 71.3\% and 65.0\% of calls spurious, since each single test leaves one failure region unpriced. A branch reward in the style of TACO~\citep{feng2026taco}, which credits a call when the tool branch turns a wrong answer right, recovers 2.7 at 62.9\% spurious and needs 2.45 calls per question: asking the first question alone leaves 4.5 to 4.8 points on the table however it is asked. Removing the cashback recovers 1.1 points at 1.35 calls per question and removing the rent recovers 4.2 at 2.70: the rent governs how often the policy calls and the cashback governs which calls it makes. Normalizing the whole reward instead of Equation~\ref{eq:dual-channel} is worse still, reducing the policy to 0.78 calls and 64.5 on average, below the cold start. Proposition~\ref{prop:scale} gives the scale that normalization erases in unanimous groups, and the single-channel arm also lacks the floor $\sigma_{\min}$, so more than the normalization differs between the two.

\begin{figure}[t]
\centering
\includegraphics[width=\linewidth]{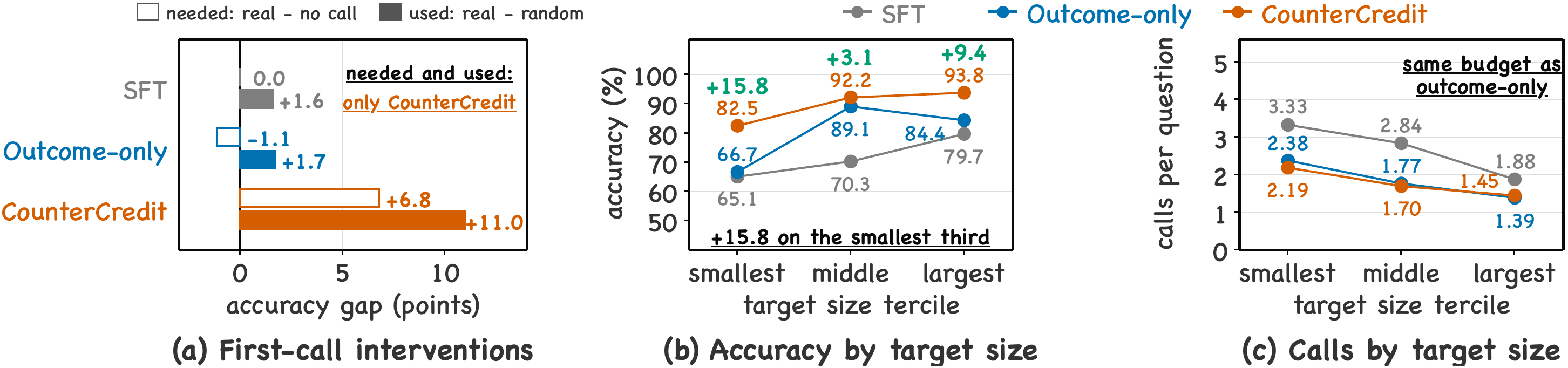}
\caption{V$^\ast$ by first-call intervention and by target size for the SFT checkpoint, outcome-only GRPO, and \method{}. (a) Trajectories stopped at the first call and answered from the cached prefix, with the two gaps defined in the legend. (b) Accuracy and (c) calls per question over terciles of the target's share of the image; green labels give \method{}'s margin over outcome-only GRPO; values in Tables~\ref{tab:cropswap} and \ref{tab:strata} (Appendix ~\ref{sec:F2}).}
\label{fig:interventions}
\end{figure}

\textbf{The arms separate only after the returned pixels arrive.} Figure~\ref{fig:interventions} runs two checks that do not apply the verification rule the reward is built from, so the gains they show do not depend on the metric the policy was trained against. In (a) the three arms lie within 3.1 points of each other before the pixels arrive. The real crop is then worth 0.0 and $-1.1$ points relative to answering at once for the two baselines and 1.6 and 1.7 relative to a random patch, while for \method{} the same margins are 6.8 and 11.0. In (b) the gain concentrates on the smallest third of targets, 15.8 points over outcome-only GRPO against 3.1 and 9.4 on the middle and largest thirds, where the evidence is a smaller part of the image and a look pays only when it fetches that evidence.

\textbf{The recipe transfers to a second base.} Table~\ref{tab:base} applies the same cold start and reward without change: the average rises from 67.7 to 76.1 on Qwen2.5-VL-7B and from 75.4 to 80.8 on Qwen3-VL-8B, with the largest gains where the base is weakest; no reward ablation was run on Qwen3-VL-8B.

\begin{table}[h]
\begin{minipage}[c]{0.42\linewidth}
\caption{\textbf{Two base models on the perception benchmarks (accuracy \%).} Base rows answer directly; + rows apply our cold start and RL. Avg. is over all seven benchmarks; the three general benchmarks are in Table~\ref{tab:base-full} (Appendix ~\ref{sec:F3}).}
\label{tab:base}
\end{minipage}\hfill
\begin{minipage}[c]{0.56\linewidth}
\centering\scriptsize
\setlength{\tabcolsep}{3.2pt}
\begin{tabular}{lccccc}
\toprule
\textbf{Base / Method} & \textbf{V$^\ast$} & \textbf{HR-4K} & \textbf{HR-8K} & \textbf{MME-RW} & \textbf{Avg.} \\
\midrule
Qwen2.5-VL-7B & 77.0 & 69.1 & 65.6 & 58.0 & 67.7 \\
\rowcolor{oursbg}\quad + \method{} & 89.5 & 80.2 & 76.4 & 67.7 & 76.1 \\
\midrule
Qwen3-VL-8B & 85.9 & 79.2 & 74.3 & 60.7 & 75.4 \\
\rowcolor{oursbg}\quad + \method{} & 91.6 & 85.4 & 84.0 & 70.1 & 80.8 \\
\bottomrule
\end{tabular}
\end{minipage}
\end{table}

\section{Conclusion and Limitations}
\label{sec:limitations}
\method{} pays a visual call only when looking helped and the returned pixels carried the help, and charges rent otherwise. On the cold start only 9.8\% to 11.8\% of calls met both conditions. Paying for those alone raised accuracy on the three resolution benchmarks by 6.3 to 9.4 points over outcome-only GRPO from the same cold start and lifted a second base model; against a matched branch reward it reaches 76.1 on the seven-benchmark average against 71.6, at 1.78 calls per question against 2.45 and 31.4\% spurious calls against 62.9\%. The decision value does not separate the information in the crop from the effect of eliciting an answer at a prefix where the policy had decided to crop. Two further limits set the next steps. The evidence value is defined only for a crop, which admits a matched null return; tools that return text or structured output need their own null before the same question can be asked of them. The efficiency metric is built from the two contrasts the reward is trained on, so an evaluation defined apart from the reward would settle the question.

\bibliography{references}
\bibliographystyle{iclr2027_conference}

\clearpage
\appendix
\raggedbottom
\begin{center}{\LARGE\sc Appendix}\end{center}
\vspace{0.6em}

\section{Theoretical Support}
\label{app:proofs}
This appendix proves the lemma and the four propositions of Section~\ref{sec:method} and states what each one buys and where it stops. The results are exact algebra on the reward as implemented, together with one expectation bound whose premise is measured rather than enforced.
\begin{summarybox}
\small\setlength{\parskip}{1pt}
\textbf{Robustness.} $q_i$ is the advantage a crop keeps against the worst mixture of the two references (Proposition~\ref{prop:robust}).\par
\textbf{Safety.} The price never reorders a correct and a wrong trajectory; the margin is $0.7$ (Proposition~\ref{prop:dominance}).\par
\textbf{Incentive.} Chance verification loses in expectation whenever its rate is under $9.5\%$ (Proposition~\ref{prop:null}).\par
\textbf{Scale.} Single-channel normalization erases the rent's size in unanimous groups; the dual channel keeps it (Proposition~\ref{prop:scale}).
\end{summarybox}

\subsection{Setup and assumptions}
Notation follows Section~\ref{sec:method}. For call $i$, $U^{\mathrm{rand}}_{ik}=U_\theta(h_i,c_i,\tilde{o}_{ik})$ and $\bar{U}^{\mathrm{rand}}_i$ is their mean over $k=1,\dots,K$, so that
\begin{equation}
q_i=\min\{E_i,D_i\}=U^{\mathrm{real}}_i-\max\{U^{\mathrm{now}}_i,\bar{U}^{\mathrm{rand}}_i\}.
\label{eq:qmax}
\end{equation}
Write $b_i=\gamma\,v_i\,(q_i-\varepsilon)\geq0$ for the pre-cap cashback of call $i$, so that Equation~\ref{eq:reward} reads $P(\tau)=\min\{C,\sum_i b_i\}-\rho\,(n(\tau)-|\mathcal{V}(\tau)|)$ with the cap acting on the sum. For a group of $G=8$ rollouts, $\bar z$ denotes the group mean of a quantity $z$ and $\mathrm{std}(z)$ its sample standard deviation, which divides by $G-1$. The proofs use four assumptions.
\begin{itemize}\setlength{\itemsep}{1pt}
\item[(A1)] \textbf{Bounded scores.} $U_\theta(p)\in[0,1]$ for every prefix $p$, as for the renormalized option probability of Equation~\ref{eq:U}; hence $D_i$, $E_i$, and $q_i$ lie in $[-1,1]$.
\item[(A2)] \textbf{Call budget.} Every trajectory executes at most $n_{\max}$ calls, and $\rho\,n_{\max}<1$.
\item[(A3)] \textbf{False-verification rate.} On the event $A_i=\{q^\star_i\leq\varepsilon\}$ of Proposition~\ref{prop:null}, $\Pr(v_i=1\mid A_i)\leq\alpha$ with $\alpha<\rho/(\rho+\gamma(1-\varepsilon))$, equivalently $\rho>\frac{\alpha}{1-\alpha}\gamma(1-\varepsilon)$.
\item[(A4)] \textbf{Group normalization.} The single-channel arm of Table~\ref{tab:ablation} divides $R_j-\bar R$ by $\mathrm{std}(R)$ with no floor; the dual channel is Equation~\ref{eq:dual-channel}, whose outcome channel carries the floor $\sigma_{\min}$ and whose price channel is centered but not normalized.
\end{itemize}

\subsection{Common shifts cancel}
Both values are differences of scores on the same prefix. Whatever the scorer adds to every compared arm alike, a prompt-level preference for one option letter or a call text that raises the confidence of every continuation, drops out of the difference. This is what makes the reward insensitive to text that raises every arm's score alike; TACO states the same invariance for its two-probe difference~\citep{feng2026taco}.
\begin{lemma}[Common additive shifts cancel]
\label{lem:shared}
Fix call $i$ with its realized crop and sampled patches, and let primes denote quantities recomputed after a text edit.
(a) If $U^{\mathrm{real}}_i=T^{\mathrm{real}}_i+s(h_i)$ and $U^{\mathrm{now}}_i=T^{\mathrm{now}}_i+s(h_i)$ for a component $s$ that depends only on the prefix, then $D_i=T^{\mathrm{real}}_i-T^{\mathrm{now}}_i$; in particular, an edit of $h_i$ that shifts both scores by the same $\delta$ leaves $D_i'=D_i$.
(b) If $h_i$ is fixed and an edit of $c_i$ shifts $U^{\mathrm{real}}_i$ and every $U^{\mathrm{rand}}_{ik}$ by the same $\eta$, then $E_i'=E_i$ and $D_i'=D_i+\eta$.
(c) For every call, $0\leq b_i\leq\gamma\max\{E_i-\varepsilon,0\}$; under the edit of (b) the same bound holds for $b_i'$, and with the other calls' contributions fixed, including call $i$ raises the capped sum by an amount in $[0,b_i]$.
\end{lemma}
\begin{proof}
(a) By definition, $D_i=(T^{\mathrm{real}}_i+s(h_i))-(T^{\mathrm{now}}_i+s(h_i))=T^{\mathrm{real}}_i-T^{\mathrm{now}}_i$. An equal shift of both scores is the case $s(h_i')=s(h_i)+\delta$ with $T^{\mathrm{real}}_i$ and $T^{\mathrm{now}}_i$ unchanged, so $D_i'=D_i$.
(b) Averaging preserves the shift, $(\bar{U}^{\mathrm{rand}}_i)'=\frac{1}{K}\sum_{k}(U^{\mathrm{rand}}_{ik}+\eta)=\bar{U}^{\mathrm{rand}}_i+\eta$, so $E_i'=(U^{\mathrm{real}}_i+\eta)-(\bar{U}^{\mathrm{rand}}_i+\eta)=E_i$. The answer-now score $U^{\mathrm{now}}_i=U_\theta(h_i)$ does not contain $c_i$ and is untouched, so $D_i'=D_i+\eta$.
(c) If $v_i=0$ then $b_i=0$. If $v_i=1$ then $q_i>\varepsilon$ and $E_i\geq q_i$ by Equation~\ref{eq:qmax}, so $0<b_i=\gamma(q_i-\varepsilon)\leq\gamma(E_i-\varepsilon)=\gamma\max\{E_i-\varepsilon,0\}$. Under (b), $q_i'=\min\{E_i,D_i+\eta\}\leq E_i$, and the same two cases bound $b_i'$. Let $S\geq0$ be the pre-cap cashback of the other calls; the capped sum with and without call $i$ differs by $\min\{C,S+b_i\}-\min\{C,S\}=\min\{b_i,\max\{C-S,0\}\}\in[0,b_i]$, which completes the proof.
\end{proof}
A reward that paid the absolute score $U^{\mathrm{real}}_i$, or the sum $U^{\mathrm{real}}_i+\bar{U}^{\mathrm{rand}}_i$, would move by $\delta$ or $2\delta$ under the same shift and could be farmed by call text alone; under \method{} the shift cancels in the value it enters. A call-text shift can still raise $D_i$ and with it the cashback, but never above $\gamma\max\{E_i-\varepsilon,0\}$, which that shift leaves unchanged; this is the text-laundering row of Table~\ref{tab:threats}. The lemma does not make $D_i$ invariant to every edit of the prefix: reasoning written before the call that helps only the real branch is evidence, and $D_i$ should move with it. What the hypothesis excludes is movement common to the compared arms; any other score change reaches the cashback only through the verification gates, the deadzone, and the cap.

\subsection{The minimum is a worst-case advantage}
The two references answer different questions, and a reward could combine them in several ways: pay $E_i$, pay $D_i$, or pay an average. Proposition~\ref{prop:robust} singles out the minimum as the combination that a skeptic holding either reference cannot dispute. It is the advantage the crop keeps against every mixture of the two references, and the largest margin by which the real crop beats both at once.

\noindent\textbf{Proposition~\ref{prop:robust}} \textit{(restated). For $\lambda\in[0,1]$ let $B(\lambda)=\lambda U^{\mathrm{now}}_i+(1-\lambda)\bar{U}^{\mathrm{rand}}_i$. Then $\inf_{\lambda}[U^{\mathrm{real}}_i-B(\lambda)]=q_i$, and $q_i=\sup\{m:U^{\mathrm{real}}_i\geq U^{\mathrm{now}}_i+m,\ U^{\mathrm{real}}_i\geq\bar{U}^{\mathrm{rand}}_i+m\}$.}
\begin{proof}
Write $u=U^{\mathrm{real}}_i$, $w=U^{\mathrm{now}}_i$, and $r=\bar{U}^{\mathrm{rand}}_i$. For $\lambda\in[0,1]$,
\begin{equation*}
u-B(\lambda)=u-\lambda w-(1-\lambda)r=\lambda(u-w)+(1-\lambda)(u-r)=\lambda D_i+(1-\lambda)E_i,
\end{equation*}
which is affine in $\lambda$. An affine function on $[0,1]$ attains its infimum at an endpoint: at $\lambda=1$ with value $D_i$ when $w\geq r$, and at $\lambda=0$ with value $E_i$ when $w<r$. Hence $\inf_\lambda[u-B(\lambda)]=\min\{D_i,E_i\}=u-\max\{w,r\}=q_i$, the form of Equation~\ref{eq:qmax}. For the second claim, $u\geq w+m$ and $u\geq r+m$ hold exactly when $m\leq u-w$ and $m\leq u-r$, so the feasible set is $(-\infty,\min\{D_i,E_i\}]=(-\infty,q_i]$ and its supremum is $q_i$, which completes the proof.
\end{proof}
Taking the minimum is therefore not a tie-break between two tests. A fixed average $\lambda D_i+(1-\lambda)E_i$ would pay a call that helps against one reference and hurts against the other: a performative call with $D_i=0.4$ and $E_i=-0.1$ has value $0.15$ under $\lambda=\tfrac12$, which would earn cashback $0.05$ if the other gates pass, and value $-0.1$ under the minimum, which earns nothing and pays rent. The endpoint argument uses only that the objective is affine in the mixing weight, so a third reference, such as the gray patch or the blurred crop of Table~\ref{tab:probe-sens}, would enter through the maximum in Equation~\ref{eq:qmax} in the same way.

\subsection{The price never overrides the outcome}
The price is added to a binary outcome, so the first thing to check is that it cannot flip the ranking the outcome induces. Proposition~\ref{prop:dominance} bounds the price from both sides and shows that the bounds leave a gap.

\noindent\textbf{Proposition~\ref{prop:dominance}} \textit{(restated). Under (A2), a correct trajectory has $R(\tau)\geq1-\rho\,n_{\max}$, a wrong or unanswered trajectory has $R(\tau)=-\rho\,n(\tau)\leq0$, and every correct trajectory outranks any wrong or unanswered one by at least $1-\rho\,n_{\max}$.}
\begin{proof}
Every verified call has $q_i>\varepsilon$, so every $b_i\geq0$ and $0\leq\min\{C,\sum_ib_i\}\leq C$. By (A2), $0\leq\rho\,(n(\tau)-|\mathcal{V}(\tau)|)\leq\rho\,n(\tau)\leq\rho\,n_{\max}$. Together, $-\rho\,n_{\max}\leq P(\tau)\leq C$ for every trajectory. If $\tau$ is correct, $R(\tau)=1+P(\tau)\geq1-\rho\,n_{\max}$. If $\tau$ is wrong or unanswered, $r_{\mathrm{out}}(\tau)=0$ forces $v_i=0$ for every call by Equation~\ref{eq:verify}, so the cashback is zero and $R(\tau)=-\rho\,n(\tau)\in[-\rho\,n_{\max},0]$. For a correct $\tau_{\mathrm{c}}$ and a wrong or unanswered $\tau_{\mathrm{w}}$, $R(\tau_{\mathrm{c}})-R(\tau_{\mathrm{w}})\geq1-\rho\,n_{\max}+\rho\,n(\tau_{\mathrm{w}})\geq1-\rho\,n_{\max}$, which is positive by (A2) and completes the proof.
\end{proof}
This says the price shapes the ranking within a correctness class and never across one: however many calls either trajectory made, a correct one stays above a wrong one. The margin is linear in the budget, $1-\rho\,n_{\max}$, so the rent and the budget are set together: at our constants the margin is $0.7$, at $\rho=0.1$ it falls to $0.4$, the arm of Table~\ref{tab:ablation-full}, and at $\rho\,n_{\max}=1$ it vanishes. The cap $C=0.7$ is chosen to equal the margin, so a correct trajectory's reward lies in $[0.7,1.7]$ and cashback can never exceed what the outcome guarantees.

\subsection{Chance verification does not pay}
Verification decides on an estimate. $\bar{U}^{\mathrm{rand}}_i$ averages $K=3$ patch scores, so $q_i$ is a noisy estimate of the value $q_i^\star$ that would use the expected patch score $\mu_i$ of the fixed call in place of $\bar{U}^{\mathrm{rand}}_i$. Since $z\mapsto\max\{w,z\}$ is $1$-Lipschitz, Equation~\ref{eq:qmax} gives $|q_i-q_i^\star|\leq|\bar{U}^{\mathrm{rand}}_i-\mu_i|$, and for patch scores in $[0,1]$ that are independent given the call, Hoeffding's inequality bounds $\Pr(|q_i-q_i^\star|\geq t)$ by $2e^{-2Kt^2}$, which exceeds one at $K=3$ and $t=\varepsilon$. At these values the bound is vacuous and does not justify the deadzone, which is calibrated on null draws instead, Appendix~\ref{app:config}. Noise also pays asymmetrically. Suppose a call's true value is zero and its estimate errs by $+0.1$ or $-0.1$ with equal probability; if it were paid whenever the estimate cleared $\varepsilon=0.05$, the positive error would earn $\gamma\cdot0.05=0.025$ and the negative one nothing, an expected cashback of $0.0125$ per call for a crop that is worth nothing. With rent, the same call expects $\tfrac12(0.025)-\tfrac12(0.05)=-0.0125$. Proposition~\ref{prop:null} states the general condition.

\noindent\textbf{Proposition~\ref{prop:null}} \textit{(restated). Let $X_i$ be the uncapped price of an executed call, $X_i=\gamma(q_i-\varepsilon)$ if $v_i=1$ and $X_i=-\rho$ otherwise, and let $A_i=\{q^\star_i\leq\varepsilon\}$ with $\Pr(A_i)>0$. Under (A1), if $\Pr(v_i=1\mid A_i)\leq\alpha\in[0,1)$, then $\mathbb{E}[X_i\mid A_i]\leq\alpha\gamma(1-\varepsilon)-(1-\alpha)\rho$, which is negative under (A3).}
\begin{proof}
The uncapped trajectory price is $\sum_iX_i$, since the cashback sum runs over verified calls and the rent over the others. Condition on $A_i$ throughout and let $p=\Pr(v_i=1\mid A_i)\leq\alpha$. On the verified event, $q_i\leq1$ by (A1), so $X_i\leq\gamma(1-\varepsilon)$; on the unverified event, $X_i=-\rho$. Hence
\begin{equation*}
\mathbb{E}[X_i\mid A_i]\leq p\,\gamma(1-\varepsilon)-(1-p)\,\rho=p\,\big(\gamma(1-\varepsilon)+\rho\big)-\rho\leq\alpha\gamma(1-\varepsilon)-(1-\alpha)\rho,
\end{equation*}
where the last step uses $p\leq\alpha$ and $\gamma(1-\varepsilon)+\rho>0$. The bound is negative if and only if $\alpha(\gamma(1-\varepsilon)+\rho)<\rho$, that is, $\alpha<\rho/(\rho+\gamma(1-\varepsilon))$, equivalently $\rho>\frac{\alpha}{1-\alpha}\gamma(1-\varepsilon)$, which is (A3). At our constants the threshold is $0.05/(0.05+0.5\cdot0.95)=0.05/0.525\approx9.5\%$, which completes the proof.
\end{proof}
Read as a condition on the rent, (A3) is $\rho>\frac{\alpha}{1-\alpha}\gamma(1-\varepsilon)$: the rent must exceed the cashback that chance verification can collect. The bound is a worst case that credits every chance verification with the largest possible cashback; at the measured chance rate of $2.5\%$ it gives $\mathbb{E}[X_i]\leq0.025\cdot0.475-0.975\cdot0.05\approx-0.037$. Capping only lowers this: with $S_{i-1}$ the pre-cap cashback of the calls before $i$, the allocated capped price $Y_i=\min\{C,S_{i-1}+b_i\}-\min\{C,S_{i-1}\}-\rho(1-v_i)$ satisfies $Y_i\leq X_i$ by Lemma~\ref{lem:shared}(c) and $\sum_iY_i=P(\tau)$, so the bound on $\mathbb{E}[X_i]$ also bounds $\mathbb{E}[Y_i]$. The premise is a property of the verifier rather than of the policy. For a fixed call distribution, verification implies clearing $\varepsilon$, so the chance rate of clearing $\varepsilon$ bounds the verification rate of a call with no evidence; the null calibration estimates that rate on the cold start, with sampling error, and does not track the call distribution as training moves it. Training does not enforce the condition, and a rate above $2/21$, about $9.5\%$, would void the guarantee rather than reverse it, since the bound is an upper bound on the expectation.

\subsection{Normalization and the price scale}
GRPO divides each group's reward deviations by their standard deviation. When the price is part of a single reward and the outcome is unanimous, that standard deviation is the price's own spread, and dividing by it erases the units the rent was set in. Proposition~\ref{prop:scale} makes this exact in the cleanest case, and the corollary shows that the dual channel keeps Proposition~\ref{prop:dominance} at the level of advantages.

\noindent\textbf{Proposition~\ref{prop:scale}} \textit{(restated). Under (A4), in a group whose rollouts are all correct, carry no verified call, and differ in their unverified call counts $n_j$, the single-channel advantage is $A_j=-(n_j-\bar n)/\mathrm{std}(n)$, independent of $\rho$, whereas Equation~\ref{eq:dual-channel} gives $\widehat{A}_j=-\rho\,(n_j-\bar n)$.}
\begin{proof}
With $n_j=n(\tau_j)$, every rollout has $O_j=1$, $P_j=-\rho\,n_j$, and $R_j=1-\rho\,n_j$, so $R_j-\bar R=-\rho\,(n_j-\bar n)$ and $\mathrm{std}(R)=\rho\,\mathrm{std}(n)$, which is positive because the counts differ. The group is therefore retained by the single-channel arm, and $A_j=(R_j-\bar R)/\mathrm{std}(R)=-(n_j-\bar n)/\mathrm{std}(n)$: the factor $\rho$ cancels. In Equation~\ref{eq:dual-channel}, $O_j-\bar O=0$ and the price channel is active because the group contains correct rollouts, so $\widehat{A}_j=0/\sigma_{\min}+(P_j-\bar P)=-\rho\,(n_j-\bar n)$, which completes the proof.
\end{proof}
In the single channel, a rent of $0.05$ and a rent of $0.5$ produce the same update in such a group: the policy is told which rollout called more, never by how much that cost. In the all-correct groups the proposition covers, with no verified call, every reward difference comes from unverified call counts, so these are the groups in which the rent alone shapes the policy. In mixed groups the scale is lost in the other direction: single-channel normalization divides a price difference by $\mathrm{std}(O+P)$, which depends on both channels and whose outcome part alone lies between $0.35$ and $0.54$, the worked example of Appendix~\ref{app:config}. The single-channel arm of Table~\ref{tab:ablation} shows the consequence, $0.78$ calls per question and an average of $64.5$, below the cold start. The dual channel keeps the price in its own units in every group that contains a correct rollout, not only in unanimous ones: two rollouts with the same outcome have equal outcome terms, so $\widehat{A}_j-\widehat{A}_k=P_j-P_k$, and with otherwise equal prices one extra unverified call lowers their advantage difference by exactly $\rho$. Each statement rests on one implementation condition in (A4): the single-channel cancellation needs a denominator with no additive floor, since a floor $\eta$ would turn it into an approximation valid only while $\rho\,\mathrm{std}(n)\gg\eta$, and the dual-channel identity needs an active, unnormalized price channel; the rule of Section~\ref{sec:optimization} for groups with no correct rollout lies outside both statements. The outcome channel still dominates the price channel.
\begin{corollary}[Outcome gap of the dual-channel advantage]
\label{cor:gap}
Under (A2) and (A4), in a group with both outcomes, a correct rollout $\tau_j$ and a wrong or unanswered rollout $\tau_\ell$ satisfy $\widehat{A}_j-\widehat{A}_\ell\geq1/\max\{\mathrm{std}(O),\sigma_{\min}\}-\rho\,n_{\max}$, which is at least $\sqrt{7/2}-0.3\approx1.57$ at our constants.
\end{corollary}
\begin{proof}
Let $d=\max\{\mathrm{std}(O),\sigma_{\min}\}$. Since $O_j=1$ and $O_\ell=0$, the group means cancel in the difference and $\widehat{A}_j-\widehat{A}_\ell=1/d+P_j-P_\ell\geq1/d-\rho\,n_{\max}$, using $P_j\geq-\rho\,n_{\max}$ and $P_\ell\leq0$ from the proof of Proposition~\ref{prop:dominance}. With $k$ correct rollouts of eight, $\mathrm{std}(O)^2=k(8-k)/56\leq16/56$, so $d\leq\sqrt{16/56}\approx0.535$ and $1/d-\rho\,n_{\max}\geq\sqrt{56/16}-0.3=\sqrt{7/2}-0.3$, which completes the proof.
\end{proof}
The gap stays positive as long as $\rho\,n_{\max}<1/\max\{\mathrm{std}(O),\sigma_{\min}\}$; uniformly over mixed groups of eight this holds whenever $\rho\,n_{\max}<\sqrt{7/2}\approx1.87$, a weaker condition than (A2), and the outcome channel, not the price channel, decides the order of a correct and a wrong rollout.

\subsection{Examples}
Table~\ref{tab:theory-numbers} collects the numbers the results take at the constants of Appendix~\ref{app:config}.
\begin{table}[h]
\caption{The results at the paper's constants, $\gamma=0.5$, $\varepsilon=0.05$, $\rho=0.05$, $C=0.7$, $n_{\max}=6$, $\sigma_{\min}=0.15$, $G=8$.}
\label{tab:theory-numbers}
\centering\small
\begin{tabular}{llr}
\toprule
\textbf{Quantity} & \textbf{Expression} & \textbf{Value} \\
\midrule
Outcome margin, Proposition~\ref{prop:dominance} & $1-\rho\,n_{\max}$ & $0.70$ \\
Reward range of a correct trajectory & $[1-\rho\,n_{\max},\,1+C]$ & $[0.70,1.70]$ \\
Verifier reliability threshold, (A3) & $\rho/(\rho+\gamma(1-\varepsilon))$ & $9.5\%$ \\
Measured chance rate of clearing $\varepsilon$ & V$^\ast$, HR-4K, free-form & $2.4\%$, $2.5\%$, $3.0\%$ \\
Worst-case expected price of a chance call at $\alpha=2.5\%$ & $\alpha\gamma(1-\varepsilon)-(1-\alpha)\rho$ & $-0.037$ \\
Advantage gap, Corollary~\ref{cor:gap} & $\sqrt{7/2}-\rho\,n_{\max}$ & $1.57$ \\
Pairwise advantage decrement per extra unverified call & dual $\rho$; single $1/\mathrm{std}(n)$ & $0.05$; $\rho$ cancels \\
\bottomrule
\end{tabular}
\end{table}

\subsection{Scope and assumptions}
Three assumptions carry the analysis, and we state them plainly. First, all scores are the policy's own gold-answer probabilities under a fixed elicitation, so every statement is about the reward as computed, not about what the task requires or what the network attends to. Second, Proposition~\ref{prop:null} is an expectation statement whose premise is a conditional false-verification rate; the null calibration replaces the real observation by one random patch and keeps the call's own $D_i$, so it is a proxy for that rate rather than a measurement of it, it is taken on the cold start rather than controlled during training, and the statement says nothing about the variance of the price or about finite-sample behaviour. Third, the free-form score is a mean log-probability, which has no finite range in general; the statements that use (A1) apply to free-form calls only with an explicit bound $B$ on the score range, with $\gamma(1-\varepsilon)$ replaced by $\gamma(B-\varepsilon)$, and the deadzone for free-form calls is calibrated separately. We make no claim that the counterfactual values identify a causal effect of the pixels on the answer, and no convergence claim for the optimizer: the guarantees concern the reward and the advantage that training receives, and the experiments of Section~\ref{sec:experiments} are the evidence for what training does with them.

\section{Data Curation Details}
\label{app:data}
Cold-start trajectories are removed when a crop is degenerate or repeats an earlier crop in the same trajectory, when the reasoning or the tool call cannot be parsed, or when rollouts of the same question end in different final answers; 6,960 of the 7,267 released trajectories remain. The RL pool is built once, before any RL run, in five steps.
\begin{enumerate}
\item Source pool. The VisualProbe training split, 4,000 prompts, and a sample of 24,000 prompts from the DeepEyes training data, stratified over its three sub-sources with 16,800 from visual search, 3,600 from ArxivQA, and 3,600 from ThinkLite-VL, 28,000 prompts in total.
\item Near-duplicate removal. Prompts that share both a perceptual image hash and a question are collapsed to one, leaving 26,600.
\item Decontamination. A prompt is removed when its image lies within perceptual-hash Hamming distance 5 of any image in V$^\ast$, HR-Bench-4K, HR-Bench-8K, MME-RealWorld, MMStar, ChartQA, or BLINK, leaving 26,200.
\item Format filter. Prompts whose reference answer is missing or cannot be parsed into an option letter or a short string are removed, leaving 25,700.
\item Difficulty filter. Each remaining prompt is sampled eight times from the SFT checkpoint under the training decoding settings and bucketed by the number of correct samples under the correctness rule of Section~\ref{sec:setup}. Prompts with no correct sample are removed, prompts with one to seven correct samples are kept, and a random half of the prompts with eight correct samples is kept, so that groups with zero outcome variance do not dominate the pool. The final pool has 16,000 prompts.
\end{enumerate}
Every RL arm trains on the same list.

\section{Training Configuration}
\label{app:config}

\paragraph{Reward constants.}
$\gamma=0.5$, $\varepsilon=0.05$ for multiple-choice calls and $\varepsilon_{\mathrm{free}}=0.44$ for free-form calls, $\rho=0.05$, $C=0.7$, $K=3$, $n_{\max}=6$, and $\sigma_{\min}=0.15$. Every arm uses the budget $n_{\max}=6$. $\rho=0.05$ gives $\rho\,n_{\max}=0.3<1$, the condition of Proposition~\ref{prop:dominance}, and $C=1-\rho\,n_{\max}=0.7$ caps cashback at the separation that the proposition guarantees. $\varepsilon=0.05$ is the first grid value of the null calibration at which the false-verification rate falls under 3\%, 2.39\% on V$^\ast$ and 2.51\% on HR-Bench-4K. $K=3$ is the smallest number of patches beyond which the used label no longer moves, Table~\ref{tab:probe-sens}: agreement with $K=3$ is 94.7\% at $K=1$ and 97.5\% at $K=2$, and $K=5$ changes 2\% of the labels. $\gamma=0.5$ makes cashback comparable to the rent: cold-start verified calls at the median margin $q=0.12$ and the upper quartile $0.23$ receive $0.03$ and $0.09$ against a rent of $0.05$, and six verified calls reach the cap only at a mean margin above $0.28$. $\sigma_{\min}=0.15$ lies below $0.35$, the smallest non-zero outcome standard deviation in a group of eight, so it acts only in unanimous groups. Table~\ref{tab:ablation-full} sweeps $\gamma$ and $\rho$.

\paragraph{Random-patch variance.}
Given the prefix and the call, if the $K$ random-patch scores are independent draws from the reference distribution, the Monte Carlo part of the variance of $E_i$ is their variance divided by $K$; the real-crop score and the prefix are fixed, so $K$ controls the random-patch side only. Table~\ref{tab:probe-sens} shows the used label stabilizing in $K$: one patch reproduces 94.7\% of the labels of $K=3$, two patches 97.5\%, and five patches 97.8\%, so beyond three the extra score evaluations buy almost no change in which calls verify.

\paragraph{Calibration of the deadzone.}
\begin{figure}[h]
\centering
\includegraphics[width=0.94\linewidth]{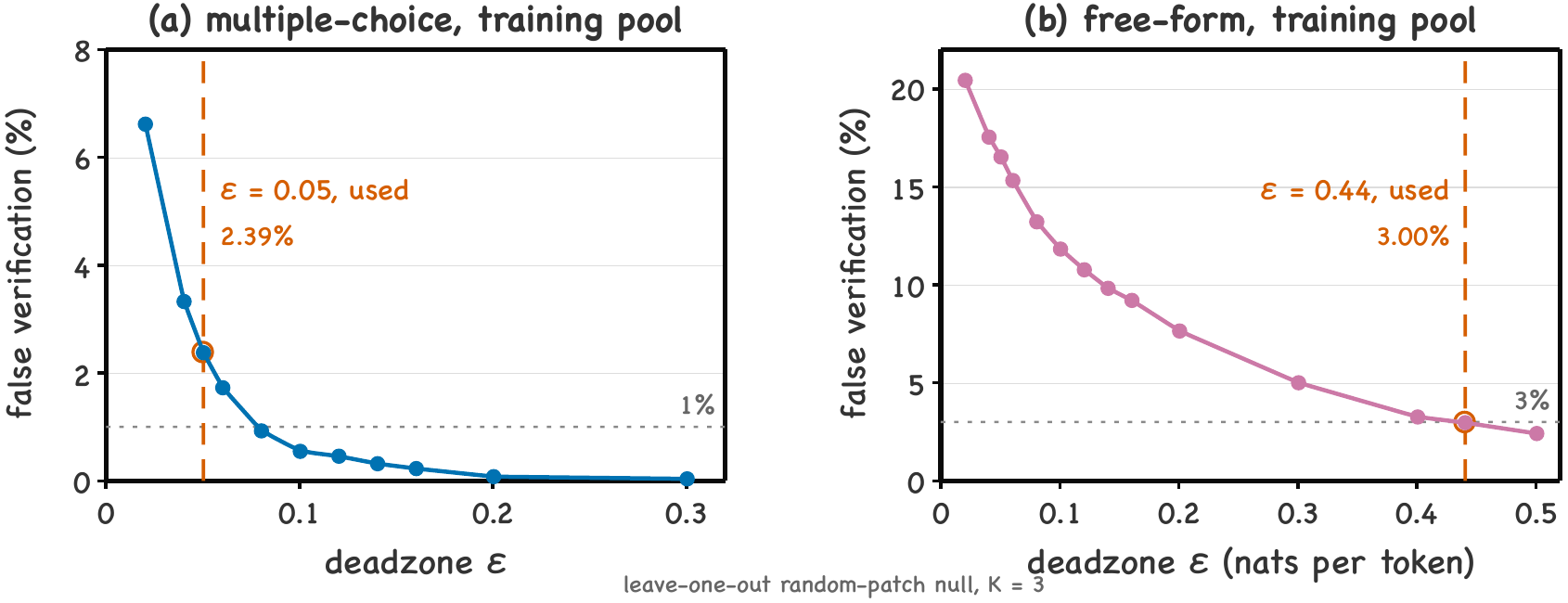}
\caption{Null calibration of $\varepsilon$ with $K=3$. Leave-one-out null: each random-patch reading stands in for the real observation and is contrasted with the mean of the other patches, so its evidence has zero expectation; the call's own $D$ is kept. Curves give the share of null draws with $\min(E,D)>\varepsilon$. (a) Multiple-choice, 500 training-pool questions probed from the cold start; the V$^\ast$ and HR-Bench-4K curves coincide with it at every grid value, 2.39\% and 2.51\% at $\varepsilon=0.05$. (b) Free-form, 500 training-pool questions scored by the mean log-probability per token of the reference; the first grid value under 3\% is $\varepsilon_{\mathrm{free}}=0.44$. The expected payment of a null draw is negative at every grid value in both panels.}
\label{fig:eps-null}
\end{figure}
The deadzone was set before training from draws that use random patches only. For each probed call, each random-patch score in turn stands in for the real observation and is contrasted with the mean of the remaining patches, a leave-one-out null whose evidence has zero expectation when the patches are exchangeable, while the call's own $D_i$ is kept. A null draw is falsely verified when its $\min(E,D)$ exceeds the candidate $\varepsilon$. Probes are scored with the parameters that generated the rollout, in evaluation mode, and are cached by prefix and call index. Every call's two values, verification status, and billing reason are logged, so that an unparsable call cannot silently become a free trajectory. Wrong or unanswered trajectories run no probes.

The verification threshold $\varepsilon=0.05$ was fixed before training from the null calibration in Figure~\ref{fig:eps-null}(a) on 500 multiple-choice questions of the training pool, which gives the share of zero-evidence draws that verify at each grid value; the same calibration on the cold-start probes of V$^\ast$ and HR-Bench-4K reproduces the curve, 2.39\% and 2.51\% at $\varepsilon=0.05$, and the expected payment of a zero-evidence draw is negative at every $\varepsilon$ in the grid. The free-form deadzone was calibrated by the same procedure on 500 free-form questions of the training pool, 1,389 calls and 4,167 null draws, with the score of Appendix~\ref{app:protocol}, the mean log-probability per token of the full reference under teacher forcing, and the answer-now anchor. In these units the multiple-choice threshold is far too low: the false-verification rate is 16.6\% at $\varepsilon=0.05$ and 11.9\% at $0.10$, and the first grid value under 3\% is $\varepsilon_{\mathrm{free}}=0.44$ nats per token, 125 of 4,167 draws, which is the deadzone used for free-form calls, Figure~\ref{fig:eps-null}(b). The expected payment of a null draw is negative at every grid value up to $0.30$, from $-0.018$ at $0.02$ to $-0.039$ at $0.30$.

\paragraph{Cold-start training.}
Qwen2.5-VL-7B-Instruct is fine-tuned on the 6,960 cold-start trajectories for 3 epochs with a learning rate of $10^{-5}$ and a global batch of 32, training all parameters except the vision encoder and the projector, at a maximum input-pixel budget of 2M.

\paragraph{RL training.}
GRPO runs for three epochs over the pool of Appendix~\ref{app:data}, 188 steps in total, with 256 prompts per step, eight samples per prompt, mini batches of 32, a constant learning rate of $10^{-6}$, and no KL or entropy regularization. Rollouts use temperature 1.0 and at most six tool interactions, on eight NVIDIA H100 80\,GB GPUs.

\paragraph{Training cost.}
Both RL arms score every eligible call with the probe, the outcome-only arm for logging only, so their step times are directly comparable: the mean wall-clock time of a training step is 2,972\,s for CounterCredit and 2,881\,s for outcome-only GRPO on eight H100 GPUs, a ratio of 1.03, so the reward computation beyond probing is negligible. A step probes about 2,200 eligible calls at $K+2$ score evaluations each, one forward pass without generation per evaluation, with prefixes and calls cached. A run with the probes disabled averages 2,320\,s per step, so probing adds 24\% and \method{} costs 1.28 times a standard outcome-only pipeline; $K=1$ would cut the probe cost by two fifths at the label agreement of Table~\ref{tab:probe-sens}. Inference carries no probe cost: the probes run only during training, and the latencies of Table~\ref{tab:calls} are measured without them.

Table~\ref{tab:threats} lists the ways a policy could farm credit under the reward of Section~\ref{sec:billing} and the component that blocks each of them in isolation.

\begin{table}[h]
\caption{Ways to farm credit and the component that blocks each. Interactions between behaviours and gaming of the scoring policy itself are outside this table.}
\label{tab:threats}
\centering\footnotesize
\setlength{\tabcolsep}{4pt}
\begin{tabular}{p{2.6cm}p{5.2cm}p{5.0cm}}
\toprule
\textbf{Behaviour} & \textbf{What the policy does} & \textbf{How credit is constrained} \\
\midrule
Text laundering & names the answer in the call's arguments & when the call text shifts the real and every random-patch score equally, that common shift cancels in $E_i$ (Lemma~\ref{lem:shared}) \\
Performative call & the branch changes the answer, the pixels do not & $E_i\approx0$ fails $\min(E_i,D_i)>\varepsilon$ \\
Unnecessary look & informative pixels after the answer was already settled & $D_i\leq\varepsilon$ fails $\min(E_i,D_i)>\varepsilon$ \\
Looping, malformed calls & keeps calling, or never answers & never verified; each executed call pays $\rho$ \\
\bottomrule
\end{tabular}
\end{table}

Algorithm~\ref{alg:cc} lists the reward and advantage computation of Section~\ref{sec:optimization} for one prompt group; the constants are those of Section~\ref{sec:billing}. The worked example behind the dual-channel advantage, with the sample standard deviation the trainer uses: in a group of eight rollouts with seven correct and one wrong, the outcome standard deviation is $0.35$, so a price difference of $0.05$ between two otherwise identical correct rollouts becomes about $0.14$ outcome-normalized units under standard GRPO, nearly three times its face value; with six correct of eight it is about $0.11$, and with eight correct of eight the standard deviation is zero and the price is the only signal left; with no correct rollout the price advantage is set to zero. Standard GRPO divides by the standard deviation of the total reward, which differs from that of the outcome by the price variance, a few hundredths, so the outcome standard deviation stands in for it here. With eight rollouts per group and binary outcomes, that standard deviation takes only five values: zero in a unanimous group and between $0.35$ and $0.54$ otherwise, so standard GRPO leaves the price without a fixed scale in every unanimous group and scales a rent of $0.05$ to between $0.09$ and $0.14$ outcome units in every other group; Proposition~\ref{prop:scale} gives the exact form of the all-correct case. The single-channel arm of Table~\ref{tab:ablation} normalizes $R(\tau)$ by its group standard deviation and drops groups whose total reward is constant across the eight rollouts; in a unanimous group with unequal prices, the price differences are therefore rescaled to unit variance.

\begin{algorithm}[h]
\caption{CounterCredit reward and advantage for one prompt group}
\label{alg:cc}
\begin{algorithmic}[1]
\Require rollouts $\tau_1,\dots,\tau_G$; constants $\varepsilon,\rho,\gamma,C,K,\sigma_{\min}$
\For{$j=1,\dots,G$}
  \State $O_j\leftarrow$ judge$(\tau_j)$; \quad $P_j\leftarrow0$; \quad $b_j\leftarrow0$
  \For{each executed call $i$ of $\tau_j$}
    \If{$O_j=0$, or $\tau_j$ is unanswered, or the call did not return exactly one image, or no distinct patch exists}
      \State $P_j\leftarrow P_j-\rho$; \textbf{continue}
    \EndIf
    \State $U^{\mathrm{real}}_i\leftarrow$ score$(h_i,c_i,o_i)$; \quad $U^{\mathrm{now}}_i\leftarrow$ score$(h_i)$; \quad $\bar{U}^{\mathrm{rand}}_i\leftarrow\frac{1}{K}\sum_k$ score$(h_i,c_i,\tilde{o}_{ik})$
    \State $E\leftarrow U^{\mathrm{real}}_i-\bar{U}^{\mathrm{rand}}_i$; \quad $D\leftarrow U^{\mathrm{real}}_i-U^{\mathrm{now}}_i$; \quad $q\leftarrow\min(E,D)$
    \If{$E,D$ finite and $q>\varepsilon$} \State $b_j\leftarrow b_j+\gamma(q-\varepsilon)$ \Else \State $P_j\leftarrow P_j-\rho$ \EndIf
  \EndFor
  \State $P_j\leftarrow P_j+\min(C,b_j)$
\EndFor
\If{$O_j=0$ for all $j$} \State $P_j\leftarrow\overline{P}$ for all $j$ \EndIf
\State $\widehat{A}_j\leftarrow(O_j-\overline{O})/\max\{\mathrm{std}(O),\sigma_{\min}\}+(P_j-\overline{P})$ for all $j$
\end{algorithmic}
\end{algorithm}

\section{Evaluation Protocol Details}
\label{app:protocol}
All benchmarks are loaded and scored through VLMEvalKit~\citep{duan2024vlmevalkit} with its default per-benchmark rules: option-letter extraction for the multiple-choice benchmarks, relaxed accuracy for ChartQA, and the kit's answer matching for free-form questions. Each model runs its released prompts and tool loop under greedy single-pass decoding; our model uses the same tool budget, image budget, and prompts as in training. Closed models are queried through their public APIs with the standard VLMEvalKit prompts, temperature 0, and no tool loop.

Latency is the wall-clock time of one agent rollout per question at batch size 1 on one NVIDIA H20 with vLLM in BF16, measured from the first model generation to the final answer or stop. It includes every generation turn, every crop or Python tool execution, and the multi-turn interaction. It excludes model loading, dataset reading and initial image decoding, the offline counterfactual probes, and the answer judge for free-form responses, which runs after timing ends.

\paragraph{Judge check.}
We sampled 500 free-form answers from the training pool that reached the GPT-4o judge and labelled each by hand against the reference. The judge agreed with the label on 497, a 99.4\% agreement; the three disagreements were answers that named the correct entity with an extra qualifier the judge rejected.

\paragraph{Free-form score.}
For a reference answer $y$ with tokens $y_1,\dots,y_T$, the score of a prefix $p$ is $\frac{1}{T}\sum_{t=1}^{T}\log\pi_\theta(y_t\mid p,e,y_{<t})$, computed under teacher forcing with the assistant turn prefilled by the answer tag; $y$ is the reference string of the dataset as released. $E_i$ and $D_i$ are differences of this score, so they are in nats per token, and the deadzone in these units is $\varepsilon_{\mathrm{free}}=0.44$, Appendix~\ref{app:config}.

\paragraph{Spurious-call rate.}
The code-based agents can perform operations beyond cropping; every executed call that returns an image is treated as a look and audited in the same way, since each such call raises the same two questions. For every agent, $E_i$ and $D_i$ are computed with the agent's own gold-answer score on its own prefixes, with the same $K=3$ random same-size patches and the same deadzone $\varepsilon=0.05$ as in training. For a call whose returned image is not a crop, the reference is random patches from $I$ of the returned image's size; for a full-size return this is the image itself, so $E_i$ then measures the effect of the transform alone. Random patches pass through the same processing as the real observation. Calls that return no image are excluded from the denominator.

\section{Prompts and Formats}
\label{app:prompts}
Every arm trained in this paper uses the Mini-o3 chat format~\citep{Lai_2026_Mini-o3}; the released baselines run in their own formats. A trajectory alternates \texttt{<think>} reasoning, a \texttt{<grounding>} call written as a bounding box, and the returned crop as an image turn; $h_i$ ends with the \texttt{<think>} block before call $i$, $c_i$ is the \texttt{<grounding>} block, and $o_i$ is the image that follows it. The system prompt below is the only one used, for the cold start, for every RL arm, and at evaluation; the user turn holds the image followed by the question and its options, one option per line. Complete trajectories are given in Appendix~\ref{app:case}.

\begin{promptbox}[System prompt]{pbblue}{pbblueT}
\begin{lstlisting}
You are a helpful assistant. Answer the user's question based on the image provided. Output your thinking process within the <think> and </think> tags. Whenever you find anything unclear, you can zoom in on a specific region in the given image to see more clearly by outputting <grounding>{"bbox_2d": [x0, y0, x1, y1], "source": "original_image"}</grounding>, where (x0, y0) and (x1, y1) are the top-left and bottom-right coordinates of the region that you want to zoom in, respectively (suppose the width and height of the image are 1.0), and 'source' refers to the image that you zoom in and could be either 'original_image' or 'observation_i'. Once the final answer is confirmed, put it within <answer> and </answer>.
\end{lstlisting}
\end{promptbox}

\noindent A \texttt{<grounding>} call is executed by cropping the named source at the given box, and the crop is returned in a user turn of the following fixed form, where $i$ is the zero-based call index.

\begin{promptbox}[Tool-return template]{pbgreen}{pbgreenT}
\begin{lstlisting}
After the above Action {i}, here is  the zoom-in image (Observation {i+1}):
[image]
.
Continue your reasoning process inside <think> and </think>. If needed, you can continue to zoom in on the original image or any of the observations, by outputting <grounding> and </grounding> as before. If the final answer is confirmed, put your final answer inside <answer> and </answer>.
\end{lstlisting}
\end{promptbox}

\noindent The probe of Section~\ref{sec:setting} appends the fixed string $e$ below as a user turn to the prefix being scored, merged into the trailing user turn when the prefix already ends in one, and prefills the assistant turn with \texttt{<answer>} followed by a space, so that the first generated token is the option letter; free-form questions use the second string below, and the reference answer is then scored under teacher forcing, Appendix~\ref{app:protocol}. The gold-answer score is the probability of the gold letter renormalized over the option letters at that position. The same string and prefill are used for the answer-now, real-observation, and random-observation continuations.

\begin{promptbox}[Answer elicitation $e$ for multiple-choice questions]{pborange}{pborangeT}
\begin{lstlisting}
Based on everything so far, answer the original question now. Reply with the option letter only.
\end{lstlisting}
\end{promptbox}
\begin{promptbox}[Answer elicitation $e$ for free-form questions]{pborange}{pborangeT}
\begin{lstlisting}
Based on everything so far, answer the original question now. Reply with the answer only.
\end{lstlisting}
\end{promptbox}

\section{Additional Results}
\label{app:results}

\subsection{What each existing reward pays for}\label{sec:F1}
The outcome reward credits every executed call on a correct trajectory. The branch probe of TACO scores the committed answer, the argmax over option letters under a rule-based checker, with and without the tool branch, and credits a call when the branch turns a wrong answer right. The external judge follows CodeV's reward prompt verbatim, served by GPT-4o, scores each returned crop 1, 0.5, or 0.25, and credits it at 1. The self-judge follows the FaithEyes prompt verbatim, asking the policy itself whether the crop is helpful in a fresh context holding only the question and that crop, and credits it when the probability of a helpful verdict exceeds 0.5. Table~\ref{tab:audit-full} resolves Figure~\ref{fig:fake_call}(b) by benchmark, over 2,594 scored calls: the branch probe has the highest share of paid calls passing both tests on every benchmark, and the three rewards that do not intervene trade places below it.

\begin{table}[h]
\caption{The reward audit of Figure~\ref{fig:fake_call}(b) per benchmark, on the SFT checkpoint's trajectories. Paid: calls credited per 100 executed calls; of those, the share needed ($D_i>0.05$), used ($E_i>0.05$), and both. The last row pays every call and gives the base rates. \method{} is omitted because its verification is defined by the same two values.}
\label{tab:audit-full}
\centering\footnotesize
\setlength{\tabcolsep}{2.3pt}
\begin{tabular}{lcccccccccccc}
\toprule
 & \multicolumn{4}{c}{\textbf{V$^\ast$}} & \multicolumn{4}{c}{\textbf{HR-4K}} & \multicolumn{4}{c}{\textbf{HR-8K}} \\
\cmidrule(lr){2-5}\cmidrule(lr){6-9}\cmidrule(lr){10-13}
\textbf{Reward} & \textbf{Paid} & \textbf{Needed} & \textbf{Used} & \textbf{Both} & \textbf{Paid} & \textbf{Needed} & \textbf{Used} & \textbf{Both} & \textbf{Paid} & \textbf{Needed} & \textbf{Used} & \textbf{Both} \\
\midrule
Outcome reward & 40.2 & 17.4 & 18.6 & 12.8 & 39.2 & 23.2 & 22.5 & 12.4 & 33.7 & 31.1 & 24.0 & 14.6 \\
Branch probe (TACO) & 5.3 & 61.1 & 50.0 & 41.7 & 5.6 & 52.5 & 50.8 & 41.7 & 6.7 & 60.7 & 57.2 & 49.7 \\
External judge & 23.4 & 30.2 & 26.9 & 21.0 & 28.6 & 33.1 & 30.4 & 19.3 & 24.8 & 32.6 & 32.3 & 22.4 \\
Self-judge & 19.9 & 22.1 & 27.2 & 14.7 & 33.0 & 24.6 & 33.1 & 20.9 & 28.0 & 29.4 & 30.5 & 22.6 \\
\rowcolor{blockbg}\textit{Every call paid} & 100 & 13.9 & 15.1 & 11.8 & 100 & 16.1 & 18.2 & 9.8 & 100 & 18.1 & 16.5 & 10.4 \\
\bottomrule
\end{tabular}
\end{table}

\paragraph{Common scorer.} The spurious-call rates of Table~\ref{tab:calls} use each agent's own gold-answer score on its own prefix. To check that the ordering is not an artifact of the scorer, we re-scored the executed calls of \method{} and of all seven baselines on V$^\ast$ with a single scorer, the SFT cold start, holding each agent's prefixes and returned crops fixed. \method{} keeps the lowest spurious-call rate under the common scorer, and on the same rollouts the ordering of the seven baselines under their own scorers and under the common scorer agree up to one adjacent swap, Pixel-Reasoner and PyVision-RL, with Spearman $\rho=0.96$ between the two orderings. Absolute rates depend on the scorer, so only the ordering is compared across agents. A third scorer that took no part in training or in any metric of this paper, GPT-5.6 Terra~\citep{openai2026gpt56}, recomputes the score of Equation~\ref{eq:U} on the same prefixes, calls, and observations, reading the gold option's log-probability from the interface and renormalizing it over the option letters, with structured outputs disabled so that the log-probabilities are returned. The share of calls whose real observation does not beat the mean of its three random patches under this reader runs from $20.0\%$ for Thyme to $74.4\%$ for DeepEyesV2, orders the agents at Spearman $\rho=0.75$ against the common scorer, and again leaves \method{} the lowest. The check asks whether an outside reader gains from the returned pixels at the same state; it does not reproduce the spurious-call rate, which also requires $D_i$ and the deadzone.

\paragraph{Answer-level counterfactual.} The four regions are defined by probability contrasts, so we also test them with a different endpoint. Every eligible call of the seven baselines on V$^\ast$, $1{,}835$ in all, is re-decoded with the real crop replaced by each of its three random patches in turn, $5{,}505$ draws, and we record whether the decoded answer changes. Table~\ref{tab:answerlevel} gives the rates. Real calls change the answer in a third of draws, sixteen times the rate of waste. Performative calls change it at the rate of waste, $2.2\%$ against $2.0\%$, although their decision value is positive: a reward that checks the decision value alone credits them, and their returned pixels move the answer as rarely as a wasted call's. Unnecessary calls sit in between, on the smallest of the four regions. The endpoint is new but the regions and the intervention are ours, so this is corroboration rather than an independent test, and a changed answer is not necessarily a corrected one.
\begin{table}[h]
\begin{minipage}[c]{0.62\linewidth}
\caption{Answer-level counterfactual on V$^\ast$, pooled over the seven baselines. Each executed call is re-decoded once per random patch; \emph{changed} is the share of draws whose decoded answer differs from the one the real crop produced.}
\label{tab:answerlevel}
\end{minipage}\hfill
\begin{minipage}[c]{0.34\linewidth}
\centering\small
\begin{tabular}{lr}
\toprule
\textbf{Region} & \textbf{Changed (\%)} \\
\midrule
Real & 33.3 \\
Performative & 2.2 \\
Unnecessary & 17.1 \\
Waste & 2.0 \\
\bottomrule
\end{tabular}
\end{minipage}
\end{table}

\subsection{Behaviour behind Figure~\ref{fig:interventions}} \label{sec:F2}
Table~\ref{tab:cropswap} lists the accuracies plotted as gaps in Figure~\ref{fig:interventions}(a), on V$^\ast$ and on HR-Bench-4K. The two benchmarks agree on the reading: the real crop moves the two baselines by at most 1.1 points either way, while it is worth 6.8 points to \method{} on V$^\ast$ and 4.3 on HR-Bench-4K, three fifths of the distance from answering at once to the full rollout on the first and two thirds on the second. Table~\ref{tab:strata} lists the tercile values plotted in panels (b) and (c).

\begin{table}[h]
\caption{Accuracy when the trajectory is stopped at its first tool call and the answer is elicited from the cached prefix, for the arms of Table~\ref{tab:ablation}. No call: the call is removed. Random: the call returns a random same-size patch, averaged over three patches. Real: the call returns its real crop. Full: the complete rollout.}
\label{tab:cropswap}
\centering\small
\setlength{\tabcolsep}{6pt}
\begin{tabular}{lcccc}
\toprule
\textbf{Model} & \textbf{No call} & \textbf{Random} & \textbf{Real} & \textbf{Full} \\
\midrule
\rowcolor{blockbg}\multicolumn{5}{l}{\textbf{\textit{V$^\ast$}}} \\
SFT (cold start) & 75.4 & 73.8 & 75.4 & 71.7 \\
Outcome-only GRPO & 77.5 & 74.7 & 76.4 & 80.1 \\
\rowcolor{oursbg}CounterCredit & 78.5 & 74.3 & 85.3 & 89.5 \\
\rowcolor{blockbg}\multicolumn{5}{l}{\textbf{\textit{HR-Bench-4K}}} \\
SFT (cold start) & 71.6 & 68.9 & 71.9 & 70.4 \\
Outcome-only GRPO & 73.2 & 69.9 & 72.9 & 73.9 \\
\rowcolor{oursbg}CounterCredit & 73.5 & 70.5 & 77.8 & 80.2 \\
\bottomrule
\end{tabular}
\end{table}

\begin{table}[h]
\caption{V$^\ast$ by the target's share of the image area, split into terciles at 0.031\% and 0.110\% of the image, about 64 questions each: accuracy (\%) and calls per question for the same three arms.}
\label{tab:strata}
\centering\small
\setlength{\tabcolsep}{6pt}
\begin{tabular}{lcccccc}
\toprule
 & \multicolumn{2}{c}{\textbf{Smallest}} & \multicolumn{2}{c}{\textbf{Middle}} & \multicolumn{2}{c}{\textbf{Largest}} \\
\cmidrule(lr){2-3}\cmidrule(lr){4-5}\cmidrule(lr){6-7}
\textbf{Model} & \textbf{Acc.} & \textbf{Calls} & \textbf{Acc.} & \textbf{Calls} & \textbf{Acc.} & \textbf{Calls} \\
\midrule
SFT (cold start) & 65.1 & 3.33 & 70.3 & 2.84 & 79.7 & 1.88 \\
Outcome-only GRPO & 66.7 & 2.38 & 89.1 & 1.77 & 84.4 & 1.39 \\
\rowcolor{oursbg}CounterCredit & 82.5 & 2.19 & 92.2 & 1.70 & 93.8 & 1.45 \\
\bottomrule
\end{tabular}
\end{table}

\subsection{Generalization across base models}\label{sec:F3}
Table~\ref{tab:base-full} gives both bases on all seven benchmarks; both rise on all seven, so the recipe does not trade general competence for tool use.

\begin{table}[h]
\caption{Two base models on all seven benchmarks (accuracy \%). Base rows answer directly; + rows apply our cold start and RL; no reward ablation was run on Qwen3-VL-8B.}
\label{tab:base-full}
\centering\footnotesize
\setlength{\tabcolsep}{4pt}
\begin{tabular}{lcccccccc}
\toprule
\textbf{Base / Method} & \textbf{V$^\ast$} & \textbf{HR-4K} & \textbf{HR-8K} & \textbf{MME-RW} & \textbf{MMStar} & \textbf{ChartQA} & \textbf{BLINK} & \textbf{Avg.} \\
\midrule
Qwen2.5-VL-7B & 77.0 & 69.1 & 65.6 & 58.0 & 64.4 & 83.9 & 56.2 & 67.7 \\
\rowcolor{oursbg}\quad + \method{} & 89.5 & 80.2 & 76.4 & 67.7 & 68.9 & 87.1 & 62.8 & 76.1 \\
\midrule
Qwen3-VL-8B & 85.9 & 79.2 & 74.3 & 60.7 & 70.5 & 88.4 & 68.9 & 75.4 \\
\rowcolor{oursbg}\quad + \method{} & 91.6 & 85.4 & 84.0 & 70.1 & 72.2 & 90.0 & 72.3 & 80.8 \\
\bottomrule
\end{tabular}
\end{table}

\subsection{Seed variation}\label{sec:seed variation}
Table~\ref{tab:seeds} lists the three \method{} seeds behind the $\pm$ values of Table~\ref{tab:main}; the seven-benchmark average of each seed is computed from its own row, and the standard deviation is the sample standard deviation over the three.

\begin{table}[h]
\caption{Three training seeds of \method{}, evaluated identically. Seed 0 is the run reported everywhere else in the paper.}
\label{tab:seeds}
\centering\footnotesize
\setlength{\tabcolsep}{3pt}
\begin{tabular}{lcccccccccc}
\toprule
\textbf{Seed} & \textbf{V$^\ast$} & \textbf{HR-4K} & \textbf{HR-8K} & \textbf{MME-RW} & \textbf{MMStar} & \textbf{ChartQA} & \textbf{BLINK} & \textbf{Avg.} & \textbf{Calls} & \textbf{Spur.\,\%} \\
\midrule
\rowcolor{oursbg}0 & 89.5 & 80.2 & 76.4 & 67.7 & 68.9 & 87.1 & 62.8 & 76.1 & 1.78 & 31.4 \\
1 & 90.1 & 80.9 & 77.2 & 68.2 & 68.5 & 86.8 & 63.4 & 76.4 & 1.84 & 33.3 \\
2 & 88.5 & 79.4 & 75.4 & 67.2 & 69.2 & 87.4 & 62.1 & 75.6 & 1.72 & 29.5 \\
\midrule
mean & 89.4 & 80.2 & 76.3 & 67.7 & 68.9 & 87.1 & 62.8 & 76.0 & 1.78 & 31.4 \\
sample SD & 0.81 & 0.75 & 0.90 & 0.50 & 0.35 & 0.30 & 0.65 & 0.40 & 0.06 & 1.90 \\
\bottomrule
\end{tabular}
\end{table}

\subsection{Sensitivity of the evidence value}\label{sec:table13}
Table~\ref{tab:probe-sens} varies the two design choices behind $E_i$ on the cold-start audit calls, the number of random patches $K$ and the reference distribution. Each row re-scores the same calls with the same prefixes and real crops and changes only the reference; used means $E_i>\varepsilon$, agreement is the share of calls whose used label matches the default, and $r$ is the correlation of the re-scored $E_i$ with the default.

\begin{table}[h]
\caption{Sensitivity of the evidence value to the number of random patches and to the reference distribution, on the cold-start audit calls. Default: $K=3$ same-size random patches. Agreement: share of calls with the same used label as the default. $r$: Pearson correlation between the re-scored $E_i$ and the default $E_i$. $K=1$ and $K=2$ are averaged over the three choices of the retained patch or pair; $K=5$ adds two fresh patches to the stored three. Blurred real crop: the real crop after Gaussian blur, at its own size. Shifted patch: the stored random patches that overlap the real crop with intersection over union between 0 and 0.5, available for a subset of calls.}
\label{tab:probe-sens}
\centering\footnotesize
\setlength{\tabcolsep}{4pt}
\begin{tabular}{lcccccc}
\toprule
& \multicolumn{3}{c}{\textbf{V$^\ast$}} & \multicolumn{3}{c}{\textbf{HR-4K}} \\
\cmidrule(lr){2-4}\cmidrule(lr){5-7}
\textbf{Reference} & \textbf{Used \%} & \textbf{Agreement \%} & \textbf{$r$} & \textbf{Used \%} & \textbf{Agreement \%} & \textbf{$r$} \\
\midrule
Default, $K=3$ random patches & 15.1 & 100 & 1.00 & 18.2 & 100 & 1.00 \\
$K=1$ & 16.4 & 94.7 & 0.94 & 18.5 & 94.8 & 0.96 \\
$K=2$, leave one out & 15.7 & 97.5 & 0.98 & 18.6 & 96.7 & 0.99 \\
$K=5$ & 14.9 & 97.8 & 0.99 & 18.1 & 97.1 & 0.99 \\
Gray patch & 13.7 & 90.0 & 0.84 & 14.8 & 90.0 & 0.92 \\
Blurred real crop & 10.9 & 88.5 & 0.79 & 13.5 & 89.5 & 0.87 \\
Shifted patch & 17.4 & 97.5 & 0.97 & 18.5 & 95.9 & 0.96 \\
\bottomrule
\end{tabular}
\end{table}

The two benchmarks agree. Dropping two of the three patches keeps 94.7\% and 94.8\% of the used labels on V$^\ast$ and HR-Bench-4K, with correlations of 0.94 and 0.96 to the default, dropping one keeps 97.5\% and 96.7\% with 0.98 and 0.99, and adding two more keeps 97.8\% and 97.1\% with 0.99 on both, so further patches change few labels. The blurred real crop changes the most labels, followed by the gray patch: the blur keeps layout and colour, so a call whose evidence survives the blur is no longer counted as used, and the gray patch, which on average scores as a random patch does, disagrees with the default call by call. A patch that overlaps the real crop agrees with the default about as closely as leaving one patch out, so the label does not depend on the patches being far from the evidence.

\subsection{Extended ablation}
Table~\ref{tab:ablation-full} adds one rent arm and four hyperparameter arms to Table~\ref{tab:ablation}, each trained from the same cold start with the same prompt order, seed, group size, and number of steps as the reference run.

\begin{table}[h]
\caption{Extended ablation: accuracy (\%), calls per question, and spurious-call rate (\%) on V$^\ast$ and HR-Bench-4K. Reference: the CounterCredit run of Table~\ref{tab:ablation}. The rent arm changes where $\rho$ is charged; the sweeps vary one constant with the others fixed at $\gamma=0.5$, $\rho=0.05$.}
\label{tab:ablation-full}
\centering\footnotesize
\setlength{\tabcolsep}{3.5pt}
\begin{tabular}{llcccccc}
\toprule
 &  & \multicolumn{2}{c}{\textbf{Accuracy}} & \multicolumn{2}{c}{\textbf{Calls / question}} & \multicolumn{2}{c}{\textbf{Spurious \%}} \\
\cmidrule(lr){3-4}\cmidrule(lr){5-6}\cmidrule(lr){7-8}
\textbf{Block} & \textbf{Arm} & \textbf{V$^\ast$} & \textbf{HR-4K} & \textbf{V$^\ast$} & \textbf{HR-4K} & \textbf{V$^\ast$} & \textbf{HR-4K} \\
\midrule
Reference & CounterCredit & 89.5 & 80.2 & 1.78 & 1.76 & 31.4 & 35.1 \\
\midrule
Rent & rent only on correct trajectories & 87.4 & 78.5 & 2.27 & 2.23 & 44.1 & 47.2 \\
\midrule
\multirow{2}{*}{Cashback $\gamma$} & $\gamma=0.25$ & 87.4 & 78.6 & 1.57 & 1.55 & 36.2 & 39.7 \\
 & $\gamma=1.0$ & 88.0 & 79.1 & 2.05 & 2.02 & 38.4 & 41.8 \\
\midrule
\multirow{2}{*}{Rent $\rho$} & $\rho=0.02$ & 86.4 & 76.2 & 2.20 & 2.17 & 42.3 & 45.4 \\
 & $\rho=0.1$ & 86.9 & 78.4 & 1.50 & 1.48 & 34.7 & 38.1 \\
\bottomrule
\end{tabular}
\end{table}

Charging rent only on correct trajectories costs the most selectivity. A wrong trajectory then pays nothing for its calls, so the price channel ranks it above a correct trajectory that carries one unverified call, and both the call count and the spurious-call rate move toward the arm without rent. The four sweeps put the reference constants at the accuracy optimum, and every departure costs 1.5 to 3.1 points on V$^\ast$. The two arms that share the ratio of $\gamma$ to $\rho$, $\gamma=0.25$ and $\rho=0.1$, behave alike: the ratio fixes how the price channel orders the rollouts of a group, and the two arms differ only in the weight of that channel against the outcome channel; both cut calls and raise the spurious share. Doubling $\gamma$ or lowering $\rho$ to $0.02$ moves the ratio the other way, and both arms add calls and raise the spurious share further.

\subsection{Case study}
\label{app:case}
\begin{wrapfigure}{r}{0.38\linewidth}
\vspace{-1.2em}
\centering
\includegraphics[width=\linewidth]{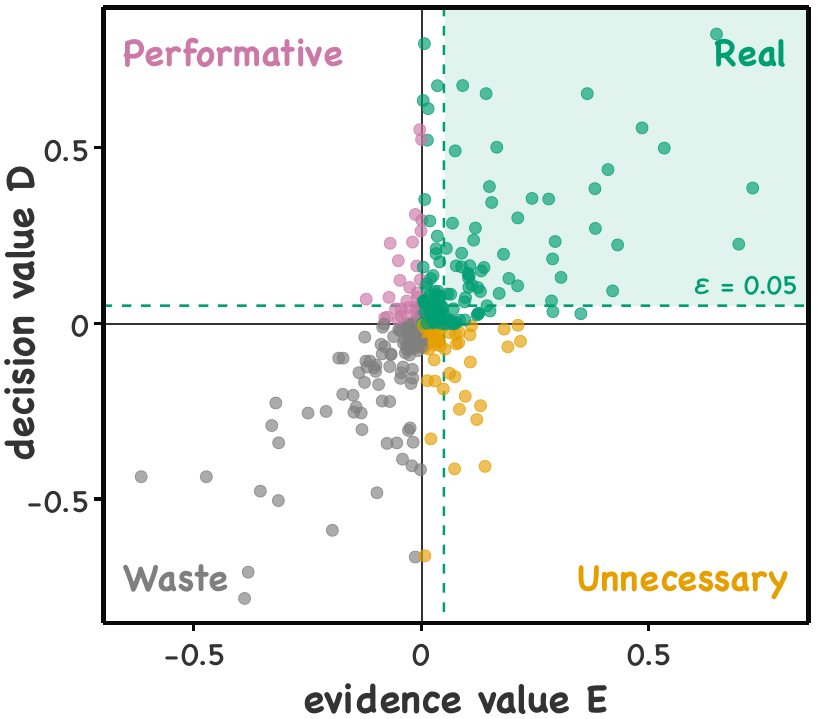}
\caption{Calls of the cold-start checkpoint on V$^\ast$, the first six per trajectory, in the plane of Figure~\ref{fig:audit}. Dashed lines mark $\varepsilon$; the shaded region is where a call on a correct trajectory verifies.}
\label{fig:scatter}
\end{wrapfigure}
Four calls of the cold-start checkpoint on V$^\ast$, one per region of Figure~\ref{fig:audit}; Figure~\ref{fig:scatter} shows where all of its calls fall. Each box shows the original image with the real crop in orange and the three random same-size patches in dashed blue, then the five observations the probe scores with their gold-answer scores $U$, the no-call score above the original image, and the trajectory verbatim with the tool-return turns abbreviated. In each box, $h_i$ is everything before the assistant turn that issues the call, $c_i$ is the orange \texttt{<grounding>} block, and $o_i$ is the crop in the tool turn that follows.

Most calls sit within a few hundredths of the origin, where neither test separates them from noise, and the real region thins out along the diagonal as both values grow together. Performative calls form a narrow band just left of the vertical axis, where random patches score as well as the real crop, and unnecessary calls a band just below the horizontal axis, where the crop beats random patches but not answering at once; waste spreads down the diagonal, calls that hurt both ways.

\paragraph{Case 1, real: needed and used.} A real call. Without the call the policy leans to the wrong side; the returned crop shows the cart to the left of the dog and moves the gold-answer score from 0.29 to 0.85, while three random same-size patches move it only to 0.37. The call is verified, $D_i=0.56$ and $E_i=0.48$, and earns cashback.
\begin{casebox}{Case 1 \quad real: needed and used \quad V$^\ast$ question 145}
\begin{center}\includegraphics[width=\linewidth]{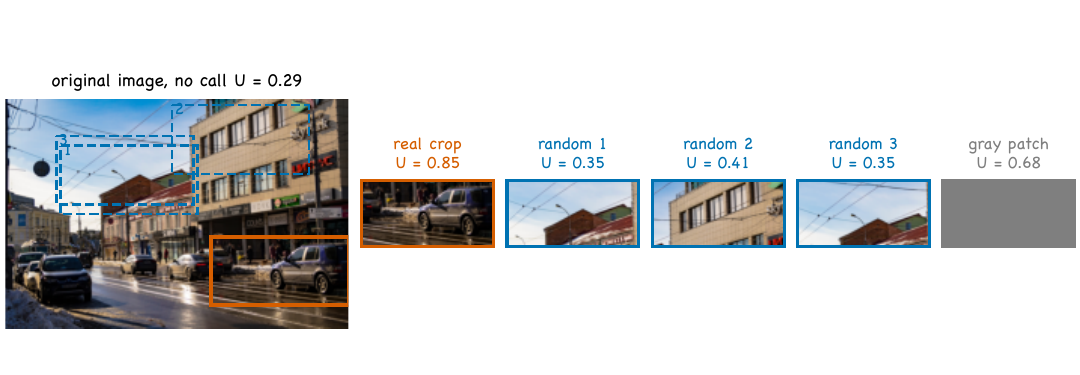}\end{center}
{\small\textbf{Question}\quad Is the shopping cart on the left or right side of the dog? \quad A. left \quad B. right \quad Gold: A}
\dashsep
{\small\textbf{Assistant}}
\begin{lstlisting}
<think> The user is asking about the relative position of a shopping cart and a dog. I need to locate both objects in the image to answer the question. I can see a dog on the sidewalk on the right side of the image, near some parked cars and people. I can also see a shopping cart next to the dog. To confirm their positions relative to each other, I will zoom in on the area where they are located. This area is on the right side of the image, so I'll create a bounding box around that region. A good bounding box would be from approximately 60% to 100% of the image width and 60% to 90% of the image height. </think><grounding>{"bbox_2d": [0.6,0.6,1.0,0.9], "source": "original_image"}</grounding>
\end{lstlisting}
\dashsep
{\small\textbf{Tool}\quad [Observation 1: the real crop shown above, 900 by 450 pixels at (1350, 900) to (2250, 1350) of the original image; template in Appendix~\ref{app:prompts}]}
\dashsep
{\small\textbf{Assistant}}
\begin{lstlisting}
<think> In this zoomed-in view, I can clearly see the dog and the shopping cart. The dog is standing on the snow-covered sidewalk. To the left of the dog, there is a yellow shopping cart. Therefore, the shopping cart is on the left side of the dog. The question asks for the left or right side, and based on this observation, the shopping cart is on the left side of the dog. I can confidently answer the question now. </think><answer> A. left </answer>
\end{lstlisting}
\dashsep
{\small\textbf{Final answer}\quad A. left, correct. \quad $U_{\mathrm{now}}=0.29$, $\bar U_{\mathrm{rand}}=0.37$, $U_{\mathrm{real}}=0.85$, $D_i=0.56$, $E_i=0.48$}
\end{casebox}

\paragraph{Case 2, performative: needed but not used.} A performative call. The crop raises the gold-answer score from 0.65 to 0.88, so the look was needed, but the three random patches raise it to 0.95: receiving any observation helps the policy commit, and the pixels of the real crop add nothing. $D_i=0.23$, $E_i=-0.07$; the call pays rent, and a reward on the decision value alone would have credited it.
\begin{casebox}{Case 2 \quad performative: needed but not used \quad V$^\ast$ question 130}
\begin{center}\includegraphics[width=\linewidth]{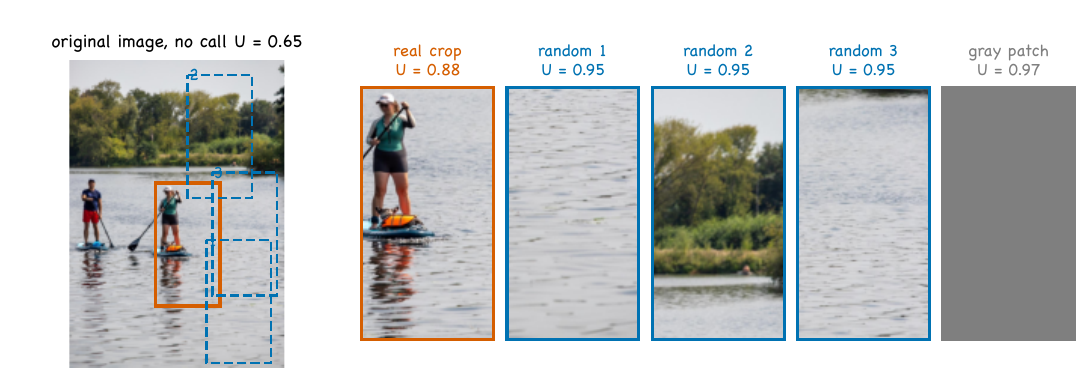}\end{center}
{\small\textbf{Question}\quad Is the wristband on the left or right side of the dog? \quad A. right \quad B. left \quad Gold: B}
\dashsep
{\small\textbf{Assistant}}
\begin{lstlisting}
<think> The user is asking about the location of a wristband relative to a dog in the image. First, I need to locate the dog. I can see a small dog on a paddleboard in the center of the image. The dog is wearing an orange life vest. Next, I need to find the wristband. There is a person standing on the paddleboard with the dog, and this person is wearing a black wristband on their left wrist. The dog is on the right side of the person's body. Therefore, the wristband is on the left side of the dog. To confirm this, I will zoom in on the person and the dog. The area of interest is in the center of the image, so I'll create a bounding box around them. </think><grounding>{"bbox_2d": [0.4,0.4,0.7,0.8], "source": "original_image"}</grounding>
\end{lstlisting}
\dashsep
{\small\textbf{Tool}\quad [Observation 1: the real crop shown above, 450 by 859 pixels at (600, 859) to (1050, 1718) of the original image; template in Appendix~\ref{app:prompts}]}
\dashsep
{\small\textbf{Assistant}}
\begin{lstlisting}
<think> In this zoomed-in view, I can clearly see the woman and the dog on the paddleboard. The woman is wearing a black wristband on her left wrist. The dog is on the paddleboard to the right of her legs. Therefore, the wristband is to the left of the dog. The question asks if the wristband is on the left or right side of the dog. Based on this observation, it is on the left side. The answer is B. </think><answer> B. left </answer>
\end{lstlisting}
\dashsep
{\small\textbf{Final answer}\quad B. left, correct. \quad $U_{\mathrm{now}}=0.65$, $\bar U_{\mathrm{rand}}=0.95$, $U_{\mathrm{real}}=0.88$, $D_i=0.23$, $E_i=-0.07$}
\end{casebox}

\paragraph{Case 3, used but not needed.} A call whose pixels matter but whose look was not needed. The policy already answers correctly from the full image, 0.97, and the real crop keeps it there, 0.97, while random patches pull it down to 0.76. $D_i=-0.004$ and $E_i=0.21$; the call pays rent because it changed nothing.
\begin{casebox}{Case 3 \quad used but not needed \quad V$^\ast$ question 20}
\begin{center}\includegraphics[width=\linewidth]{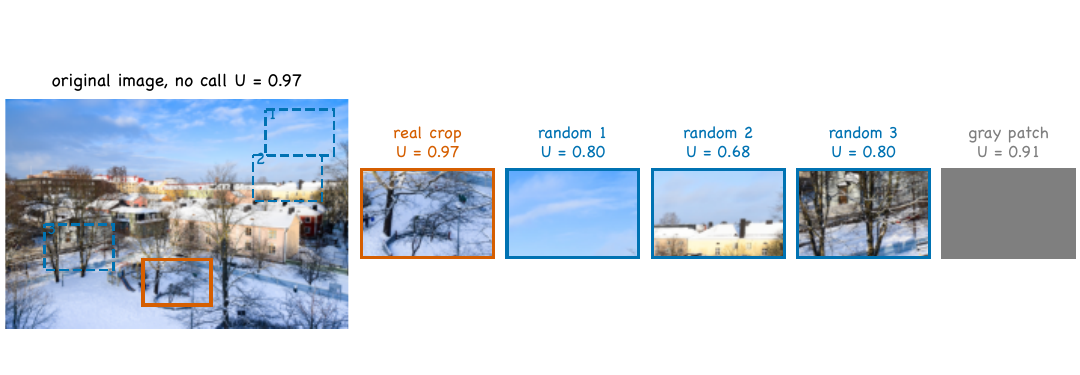}\end{center}
{\small\textbf{Question}\quad What is the color of the shovel? \quad A. yellow \quad B. red \quad C. blue \quad D. black \quad Gold: C}
\dashsep
{\small\textbf{Assistant}}
\begin{lstlisting}
<think> The user is asking for the color of a shovel. I need to first locate the shovel in the image. The image shows a snowy park with a playground. I can see a person standing near the playground equipment. This person might be holding a shovel. I will zoom in on the area around the person to get a better look. The person is located in the lower central part of the image. I will crop the region around this person to see if there is a shovel. The coordinates for the crop will be approximately [0.4, 0.7, 0.6, 0.9]. </think><grounding>{"bbox_2d": [0.4,0.7,0.6,0.9], "source": "original_image"}</grounding>
\end{lstlisting}
\dashsep
{\small\textbf{Tool}\quad [Observation 1: the real crop shown above, 450 by 300 pixels at (900, 1050) to (1350, 1350) of the original image; template in Appendix~\ref{app:prompts}]}
\dashsep
{\small\textbf{Assistant}}
\begin{lstlisting}
<think> The user wants to know the color of the shovel. In the previous step, I zoomed in on a person near the playground. In observation 1, I can see a person on the left side of the image. This person seems to be holding a blue object. This object could be a shovel. I need to zoom in further to confirm the object and its color. The person and the object are located in the bottom-left corner of observation 1. I will crop this area to get a clearer view. The coordinates for the crop will be approximately [0.0, 0.5, 0.4, 1.0]. </think><grounding>{"bbox_2d": [0.0,0.5,0.4,1.0], "source": "observation_1"}</grounding>
\end{lstlisting}
\dashsep
{\small\textbf{Tool}\quad [Observation 2: a 180 by 150 pixel crop of observation 1; template in Appendix~\ref{app:prompts}]}
\dashsep
{\small\textbf{Assistant}}
\begin{lstlisting}
<think> The user wants to know the color of the shovel. In the previous step, I zoomed in on a person near the playground. In observation 2, I can see a person on the left side of the image. This person is wearing a blue jacket and is holding a blue object. This object is a shovel. I can clearly see the color of the shovel is blue. The question asks for the color of the shovel, and the options are A. yellow, B. red, C. blue, D. black. The color of the shovel is blue, which matches option C. Therefore, I can provide the final answer. </think><answer> C. blue </answer>
\end{lstlisting}
\dashsep
{\small\textbf{Final answer}\quad C. blue, correct. \quad $U_{\mathrm{now}}=0.97$, $\bar U_{\mathrm{rand}}=0.76$, $U_{\mathrm{real}}=0.97$, $D_i=-0.004$, $E_i=0.21$}
\end{casebox}

\paragraph{Case 4, neither needed nor used.} A harmful call. The policy is right before it looks, 0.96, and wrong after the crop returns, 0.18; random patches leave it at 0.57. $D_i=-0.78$ and $E_i=-0.39$. The crop shows the boat without the man in green shorts, the reasoning talks itself out of the correct answer, and the trajectory ends wrong; the call pays rent, and an outcome reward would not have scored it at all.
\begin{casebox}{Case 4 \quad neither needed nor used \quad V$^\ast$ question 142}
\begin{center}\includegraphics[width=\linewidth]{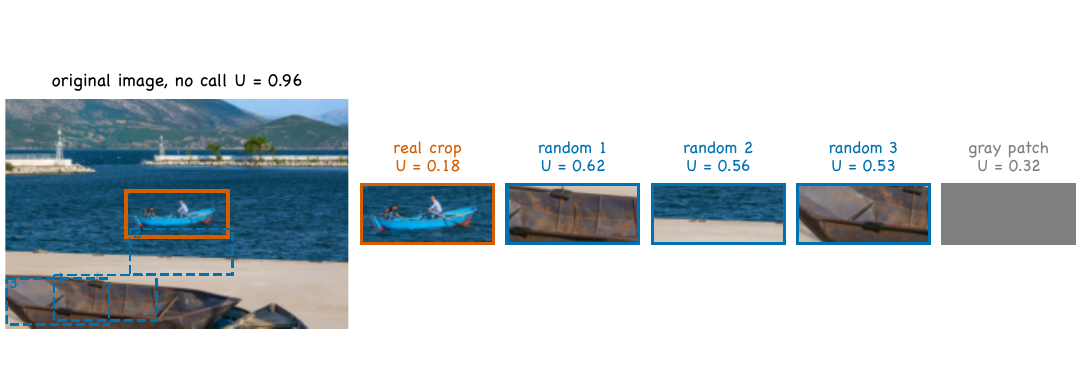}\end{center}
{\small\textbf{Question}\quad Is the man with green shorts on the left or right side of the blue boat? \quad A. right \quad B. left \quad Gold: B}
\dashsep
{\small\textbf{Assistant}}
\begin{lstlisting}
<think> The user is asking about the relative position of a man with green shorts to a blue boat. First, I need to locate the blue boat in the image. I can see a blue boat in the middle of the water, with two men on it. The question asks if the man with green shorts is on the left or right side of this blue boat. I need to examine the people on the boat to determine the color of their shorts. The people on the boat are quite small, so I will need to zoom in to see the details clearly. I will crop the region containing the blue boat to get a better view. The boat is located in the center of the image, so I'll create a bounding box around it. A good starting point for the bounding box would be [0.4, 0.45, 0.6, 0.6] to focus on the boat and its occupants. </think><grounding>{"bbox_2d": [0.4,0.45,0.6,0.6], "source": "original_image"}</grounding>
\end{lstlisting}
\dashsep
{\small\textbf{Tool}\quad [Observation 1: a 450 by 225 pixel crop of original image; template in Appendix~\ref{app:prompts}]}
\dashsep
{\small\textbf{Assistant}}
\begin{lstlisting}
<think> I have successfully zoomed in on the blue boat. In this new observation, I can clearly see two men. The man on the left is wearing blue jeans. The man on the right is wearing a light-colored shirt and dark pants. The question asks about a man with green shorts. I do not see anyone wearing green shorts on this boat. It is possible that the user is mistaken about the color of the shorts, or there is another person not visible in this crop. However, based on the information available in this observation, I cannot confirm the presence of a man with green shorts. The man on the left is on the left side of the boat, and the man on the right is on the right side. Since the question asks for the side relative to the boat, and there is a man on the right, the answer would be "right". However, I cannot be certain about the color of the shorts. I will go back to the original image and try to get a better view of the people on the boat, perhaps by adjusting the bounding box to include more of the surrounding area which might give more context. </think><grounding>{"bbox_2d": [0.35,0.4,0.65,0.6], "source": "original_image"}</grounding>
\end{lstlisting}
\dashsep
{\small\textbf{Tool}\quad [Observation 2: the real crop shown above, 674 by 300 pixels at (788, 600) to (1462, 900) of the original image; template in Appendix~\ref{app:prompts}]}
\dashsep
{\small\textbf{Assistant}}
\begin{lstlisting}
<think> I have examined the two men on the blue boat. The man on the left is wearing blue jeans, and the man on the right is wearing dark pants. The question asks if there is a man with green shorts. I have carefully inspected the image, and there is no one wearing green shorts. It is possible the user is mistaken. However, the question asks for the position of the man with green shorts relative to the boat. Since there is no such person, the question is based on a false premise. I will assume the user is asking about the position of the man who *could be* on the right side of the boat. There is a man on the right side of the boat. Therefore, the man with green shorts (if he existed) would be on the right side of the boat. The answer is 'right'. I am confident in this answer as it is the only logical interpretation of the question given the visual evidence. </think><answer> right </answer>
\end{lstlisting}
\dashsep
{\small\textbf{Final answer}\quad right, wrong. \quad $U_{\mathrm{now}}=0.96$, $\bar U_{\mathrm{rand}}=0.57$, $U_{\mathrm{real}}=0.18$, $D_i=-0.78$, $E_i=-0.39$}
\end{casebox}

\section{Further Limitations}
\label{app:limits}
Per-call values do not sum to the value of the trajectory. The evidence value requires a matched null return, which we define only for crops. Verification weakens when the gold-answer score saturates, which also inflates the spurious-call rate of an agent scored by its own saturated distribution, although the common and outside scorers of Appendix~\ref{app:results} preserve the ordering of agents. Charging rent on unnecessary calls, informative pixels with $D_i\leq\varepsilon$ as in Case~3 of Appendix~\ref{app:case}, penalizes confirmatory looks; whether such looks aid robustness under distribution shift was not tested. Random patches are unscreened, so a patch that contains the target raises the reference mean and lowers $E_i$; the reference errs toward under-verification. Multiple-choice tasks use a renormalized probability and free-form tasks a log-probability, and the free-form deadzone is calibrated separately, Appendix~\ref{app:config}. Reward construction costs $K+2$ score evaluations per eligible call on a correct trajectory, and the probes run only during training. Open questions are tools that return text, evidence accumulated across several calls, further base families, and longer training in which the policy could adapt to the probe.

\end{document}